\documentclass{article}

\newcommand{\titel}
{Adams Conditioning and Likelihood Ratio Transfer Mediated Inference}

\usepackage{stmaryrd,amssymb,amsmath,amsthm,url}

\newtheorem{theorem}{Theorem}[subsection]

\newcommand{\midd}{\:|\:}

\newcommand{\PFP}{\ensuremath{\mathit{P\!F}_{\!P}}}

\newcommand{\BA}{\ensuremath{\textit{BA}}}
\newcommand{\Md}{\ensuremath{\textit{Md}}}
\newcommand{\Sign}{\ensuremath{\textit{Sign}}}

\newcommand{\FE}{\ensuremath{\textit{FE}}}
\newcommand{\LR}{\ensuremath{\textit{LR}}}
\newcommand{\TOF}{\ensuremath{\textit{TOF}}}
\newcommand{\MOE}{\ensuremath{\textit{MOE}}}

\title{\titel}

\author{
	Jan A.\ Bergstra 	\\
\\
  {\small
	  Informatics Institute,
	  University of Amsterdam}\\
	{\small \url{email: j.a.bergstra@uva.nl, janaldertb@gmail.com}
		}
}
\date{\small{August 10, 2018}}

\begin{document}
\maketitle

\begin{abstract}
Bayesian inference as applied in a legal setting  is about belief transfer and involves a 
plurality of agents and  communication protocols. 

A forensic expert (FE) may communicate to a trier of fact (TOF) first its value of a certain 
likelihood ratio with respect to FE's belief state as represented by a probability function  on FE's proposition space. 
Subsequently FE communicates its recently acquired confirmation that a certain evidence proposition  is true. 
Then TOF performs  likelihood ratio transfer mediated reasoning thereby revising their own belief state.

The logical principles involved in likelihood transfer mediated reasoning are discussed in a setting where 
probabilistic arithmetic is done within a meadow,
and with Adams conditioning placed in a central role.
\\[4mm]
\end{abstract}

\tableofcontents
\section{Introduction}
The work in this  paper has been triggered by the following question stated in Lund \& Iyer~\cite{LundI2016}: 
why not separately communicate the two likelihoods that make up a likelihood ratio? 
A proposal for an answer to this question is given in Theorem~\ref{AdamsBayes2} and the subsequent comments. 
The question highlights the fact that Bayesian reasoning is about communication, as well as about probability.

Courtroom reasoning 
involving Bayesian inference takes the form of a protocol by means of which a forensic expert (FE) may interact 
with a trier of fact (TOF). The setup requires that both FE and TOF maintain their own space $S_\FE$, resp.\ 
$S_\TOF$ of propositions, and
that both maintain a belief state that is formalized as a probability function $P_\FE$ resp.\ $P_\TOF$ on the respective 
proposition spaces. It is customary to view such probabilities with the paradigm of subjective probability.

When a single probability function is used the model supposedly involves precise beliefs. When collections of 
probability functions are made use of, so-called non-singleton representors, 
the model admits imprecise beliefs. Imprecise beliefs
may be helpful or may even be needed when besides uncertainty, the realm of probability functions, 
also ignorance is being modelled. Modelling uncertainty with non-singleton representors is only 
done under the assumption that ignorance and uncertainty are different to such an extent that
 ignorance cannot be adequately represented by means
of the same probability function that is used for the representation of an agent's uncertainty. Comments 
regarding imprecise beliefs in connection with Bayesian inference are delayed until Paragraph~\ref{imprecise} below. 

Following a tradition initiated in forensics by Lindley (e.g.~\cite{Lindley1977}) and Evett 
(e.g. see~\cite{Evett2016}), who in turn based their 
work on the principles of subjective probability as set out by  Ramsey, de Finetti, Carnap, and Jeffrey, 
the work of several 
contemporary authors shows a commitment to the exclusive usage of precise belief states:
Berger \& Slooten~\cite{BergerSl2016}, Berger et.\ al.\ \cite{BergerBCEJ2011} and Biedermann~\cite{Biedermann2015, BiedermannV2016}. 
Independently of forensics, 
theory development concerning
precise beliefs has advanced in different directions, for instance in Diaconis \& Zabell~\cite{DiaconisZ1982}, 
Bradley~\cite{Bradley2005}, Gyenis~\cite{Gyenis2014}, and Yalcin~\cite{Yalcin2016}. I will make use of 
Bradley's presentation of Adams conditioning in~\cite{Bradley2005}.

The objective of the paper is as follows: (i) to develop a formalisation of the simplest applications of Bayesian reasoning in courtroom practice
in the tradition initiated by Lindley and Evett (the so-called Lindley framework) in 
terms of a protocol for the flow of information about likelihoods and evidence that is produced 
by an agent $B$ (e.g. FE, acting as a forensic expert as will be discussed later) so that this information 
  can be incorporated in the belief state of an agent $A$ (e.g. TOF, acting as a trier of fact), (ii) to develop a uniform description of 
  transformations of belief states based on Adams conditioning, including a new account of the fallacy of transposition of the conditional,  
  (iii) to provide an answer to the question 
  mentioned in the first line of the introduction, (iv) to draw a methodological conclusion regarding the plausibility of the requirement that 
  both TOF and FE strictly adhere to subjective probability theory based on precise probabilities: 
  this requirement is incompatible with single message reporting by the forensic expert, (v) to suggest a way out of this incompatibility,
  and finally, (vi) to do the technical work  in equational logic
  thereby providing a contribution to the development of probability theory in the context of signed meadows. 
  
 The paper is specific for the legal context in the sense that the protocol suggested below for the transfer of probabilistic has been designed with the
 purpose in mind to provide a detailed model of what transpires from a quite extensive and nevertheless 
 rather informal literature about the Lindley framework.
 The general problem how an agent may incorporate testimonial information it its belief state is more general than the particular approach of
 the Lindley framework and reaches far beyond the legal domain. 
 Both inside and outside the legal context other approaches and methods exist for that matter. This work focuses exclusively on aspects of the
 Lindley framework, which is understood as an instance of likelihood transfer mediated reasoning. As a justification for this focus I mention the fact that
 the assumption that for each agent involved all probabilities are subjective and must be precise is controversial. 
 Forensic experts for instance are inclined  to understand their results as delivering approximations of ``really existing probabilities'', delivered 
 together with a margin of error, rather than as
 mere expressions of their subjective judgements for which, by definition, the notion of an error or a margin of error 
 does not even make sense. 
 
  Initial results concerning the development of probability theory in equational logic can be found  in  
Bergstra \& Ponse~\cite{BergstraP2016} and Bergstra~\cite{Bergstra2016}. 
Meadows as an approach to arithmetic data types have arisen from the equational theory of abstract data types, 
starting with  Bergstra \& Tucker~\cite{BergstraT2007}. Signed meadows are introduced in Bergstra, Bethke \& Ponse~\cite{BBP2013}.
A complete axiomatization for the meadow of real numbers is established in~Bergstra, Bethke \& Ponse~\cite{BBP2015}, 
and is applied in the context of probabilities in  Bergstra \& Ponse~\cite{BergstraP2016}.

The paper will provide several abstract versions of elementary 
communication protocols between TOF and MOE such as are present in the so-called Lindley framework.
 Ignoring variations on the theme I 
will write as if one may refer with ``Lindley framework''
to a definite position which arose from the works of Ramsey, de Finetti, Carnap, Lindley, Evett, 
and which is now represented by authors including Berger, Biedermann, and Taroni. In each protocol
the transfer of information about likelihoods or likelihood ratios plays a role, as well 
as the processing of information thus obtained, and therefore I will speak of likelihood ratio transfer mediated
reasoning (LRTMR).\footnote{%
The central role of likelihood ratios in reporting in forensic science (which includes forensic 
science based forensic practice) is strongly emphasized in the ENFSI  guidelines (Willis et al.\ \cite{Willis2015}).} 


\subsection{Trier of fact (TOF) and mediator of evidence (MOE)}
In a forensic context the trier of fact (TOF) has the role of deciding about 
the truth of certain statements. Below such statements are referred to as hypothesis propositions. 
The TOF role may be played by a judge 
or by a jury. TOF may say ``guilty'' (or ``not guilty''). The TOF may  be in need of a science backed interpretation of available evidence. 
Providing such information is delegated to the forensic expert. In order to obtain a more generally applicable  
presentation I will speak of a mediator of evidence (MOE) rather than of a 
forensic expert (FE). TOF and MOE are the two major roles in any account of likelihood ratio transfer mediated reasoning.

\subsection{Background assumptions concerning LRTMR}
\label{WHyp}
This paper is written on the basis of certain assumptions and guidelines which merit being mentioned in advance.
\begin{enumerate}
\item The description of TOF side reasoning is  done having in mind the paradigm of subjective probability theory 
with precise  probabilities for quantifying the strength of an agent's belief. Concerning MOE no strict 
commitment to subjective probability theory is assumed.

\item We are describing abstract reasoning patterns that are based on the so-called Lindley framework. We do not claim that these abstractions 
conclusively capture the logical aspects of the Lindley framework. 
This framework is biased towards providing a tool for
as well as an explanation of TOF side reasoning. Arguably different reasoning frameworks underly the reasoning of  MOE or any other relevant agents. 
While operating in interaction different agents may yet entertain different conceptions of 
probability theory and statistics and exchange corresponding alphanumerical data. 

\item  Becoming aware of the very extensive foundational literature on probability theory is a challenge. 
I have no basis for the claim that the results in this paper are new except that I did not yet
find these results in this form elsewhere. 
More specifically I must be quite cautious with making any claims regarding the 
the philosophical and methodological adequacy, 
and of course the novelty and originality of the following collection of notions, methods, and suggestions 
which are introduced and used in the paper, and all of which I did not find in existing work, at least not in an explicit form.
\begin{enumerate}
\item the use of single likelihood Adams conditioning as a method for TOF to process an incoming update of a likelihood by 
MOE which complies with the requirements of subjective belief theory,
\item  the use of double likelihood Adams conditioning by TOF for processing an update of a likelihood ratio received from MOE. 
 In Paragraph~\ref{limitedM} we have some comments on the validity of this transformation from the perspective of subjective probability theory,
\item the use of an auxiliary and temporary additional proposition space,  in which incoming 
evidence information is processed by way of Bayesian conditioning with proposition kinetics, 
thereby producing a probabilistic figure which allows subsequent Jeffrey conditioning on TOF's main proposition space,
 thereby avoiding proposition kinetics on TOF's main propositions space,
 \item  the distinction between likelihood pair and likelihood ratio, viewing a likelihood pair as one of many possible 
 representations of a likelihood ratio,

 \item the notions of local representation independence and global representation independence 
 (for methods for processing update information on likelihood ratios),
\item  the notion of single message reporting (by MOE), and the suggestion that single message 
reporting is impossible for an MOE who strictly adheres to the principles of subjective probability theory, 
\item the notion of multiple message reporting by MOE, and the suggestion that both TOF and MOE may comply with the requirements of 
subjective probability theory, both agree to separate in time  
the transmission from  MOE to TOF of a  likelihood pair,   
from the subsequent transmission of an assertion.

\end{enumerate}
\item The transformations of probability functions as discussed in the paper are  relevant in their own right, even if the portrayed 
role of these transformations in the interaction between TOF and MOE is considered problematic, 

\item A calculus or notation for belief states amounts to no more than a supportive tool for TOF, 
and the calculus is  merely a toolkit. Fenton, Neil \& Berger\cite{FentonNB2016}  
expect that in actual application TOF will make use of automated support, and that the numerical calculation of 
actual beliefs and of belief state revisions will not become 
a task for human agents.

\end{enumerate}

\section{Probability calculus on the basis of involutive meadows}
\label{DBZ}
A meadow is a structure for numbers equipped with either a name (and one or more notations)  
for the multiplicative inverse function (inversive notation) or with a name (and one or more notations)
 for the division function
(divisive notation). Given either name with notations the other name and notations can be introduced as an abbreviation.
For instance one may start with $x^{-1}$ as a notation for the multiplicative inverse function and write $x/y $ as an abbreviation for $x \cdot y^{-1}$. 

From the perspective of forensic science it is wholly immaterial whether or not there is a 
name in one's language for a mathematical function, in this case division. 
Only when formalizing the logic the presence or absence of names acquires relevance.

Once a name and notation (say division with notation $x/y$ given arguments $x$ and $y$) 
has been introduced the question 
``what is $x/0$''
may be posed and is entitled to an answer. Assigning a value or meaning to
$x/0$ can be done in at least 6 different ways, each of which has been 
amply investigated in the mathematical and logical literature about the status the multiplicative inverse of zero. 
A straightforward idea, which is adopted in this paper, is to work
under the assumption that $ x/0=0$. This convention must not be portrayed as a non-trivial  insight about numbers and division, 
which has been overlooked by mainstream mathematics until now, so to say. 
The convention to set $x/0=0$ merely represents a definite selection (out of a number of options) 
on how to base one's logic on a formalized version of arithmetic. 
Using $x/0=0$ on top of the standard axioms for numbers 
(the axioms for a commutative ring) leads to what is called an involutive meadow in~\cite{BM2015}.

Developing precise logics for application in forensic science begins with the choice of a logic for arithmetic
as a step towards having a logic for the values serving as probabilities. 
Mathematics does not provide such logics, however, the provision of which belongs to mathematical logic. 
Working with an involutive meadow is just one option for choosing a logic of numbers. 
My preference for working with involutive meadows derives from a preference for working with 
equations and conditional equations over the explicit use of 
quantifiers which comes with the use of full first order logic.\footnote{%
I notice a remarkable absence of 
quantified formulae
 (in particular $\forall x. \Phi$ and $ \exists x. \Phi$ for a formula $\Phi$) in the forensic science literature. 
 In computer science quantifiers are used all over the place. This relative  lack of prominence of 
 quantifiers   is at first sight at odds with the frequent use of the term logic in forensics. 
 However, assuming that
 only universally quantified formulae are used, and taking into account the convention (which prevails 
 both in logic and in mathematics) to omit 
 explicit mention of quantifiers  while having universal quantification as a default 
  I entertain the view that the forensic science community implicitly shares my 
 preference for working with a fragment  of first order logic 
 rather over working with full first order logic. This fragment, however, is more expressive than equational logic, which for instance does without negation. 
 When working in equational logic the logical operators (negation, conjunction, disjunction, 
 material implication) are treated as mathematical functions. The latter convention is by no means generally accepted and 
 it comes with its own complications.} 
 
Below decimal notation will be used freely under the assumption that 
$2$ abbreviates $1+1$ and so on.\footnote{%
In Bergstra \& Ponse~\cite{BP2016b} the formalization of decimal number notation by means of
ground complete term rewriting systems, a useful shape of abstract datatype specification in 
preparation for prototyping implementations, is studied in detail.}

\subsection{Proposition spaces and probability functions}
Formulae and equations below are to be understood in the context of the specification 
	$\BA+\Md+\Sign+\PFP$ (for: Boolean algebra + meadows + sign function +
	 a probability function named $P$) taken 
	from Bergstra \& Ponse~\cite{BergstraP2016}. These equations hold for a probability 
	function with name $P$ over an event space $E$, which takes the form of a Boolean algebra of 
	events and for which  a finite collection $C$ of constants is available.\footnote{%
	For this specification a completeness theorem that was proven in Bergstra, Bethke \& Ponse~\cite{BBP2015} 
	for $\Md+\Sign$ is extended in~\cite{BergstraP2016}.
	Moreover the equational specification $\BA+\Md+\Sign+\PFP$ is extended 
	with so-called conditional values, playing the role of random variables, and expectation values in 
	Bergstra~\cite{Bergstra2016}.}
	As is common in the forensic literature I will refer to events as propositions below. Formulae, however, involve
	sentences and syntax. 
	Besides a Boolean algebra of propositions there
	is a Boolean algebra of sentences.

Throughout the paper I will use Jeffrey's $t(\bullet)$ notation for lambda abstraction with a single variable: 
given a context $t[-]$: $t(\bullet) = \lambda x. t[x]$.
\subsection{Conditional probability: variations on a theme}
	Following~\cite{BergstraP2016}   $P^0(x \midd y)$ is defined by 
	\[P^0(x \midd y) = \frac{P(x \wedge y)}{P(y)}\]
	 The
	superscript $^0$ indicates that  $P^0(x \midd y) = 0$ whenever $P(y) = 0$.
In~\cite{BergstraP2016} several other options for defining conditional probabilities in the case
that $P(y) = 0$ are taken into account are discussed. For instance, 
\[P^1(x \midd y)= 	 \displaystyle \frac{P(x \wedge y) - P(y)}{P(y)} + 1 \]
$P^1(-)$ satisfies: $P(y) = 0 \to 	P^1(x \midd y) = 1$, which fits well with material implication for two-valued logic.
 When dealing with Bayesian conditionalization safe conditional probability, written as 
$P^S(x \wedge y)$, may be helpful:
\[P^S(x \midd y)= 	 \displaystyle \frac{P(x \wedge y) -P(y) \cdot P(x)}{P(y)} + P(x)\]
The advantage of safe conditional probability is that $P^S(\bullet \midd y) = \lambda x . P^S(x \midd y)$ (using Jeffrey's 
$P^S(\bullet \midd y)$ notation as a ``dedicated'' form of lambda abstraction, as was stated above) 
 is a probability function  for all $y$, which is not the case for $P^0(\bullet \midd y)$ and for $P^1(\bullet \midd y)$.
Denoting with $\uparrow$ the ``undefined'' outcome of a function,
 the conventional notion of  a conditional probability due to Kolmogorov reads as follows:
\[P^\uparrow(x \midd y)= 	  \frac{P(x \wedge y)}{P(y)} \lhd P(y) \rhd  \uparrow.\]
Here $\displaystyle a\lhd \, b\, \rhd c$ stands for \emph{if b $\ne$ 0 then a else c}. 	
The conditional operator allows a straightforward definition in an involutive meadow:\footnote{%
In the presence of $\uparrow$ weaker equations are needed: such as 
$x \lhd 0 \rhd z = z,~x \lhd 1 \rhd z = x,~ x \lhd \uparrow \rhd z =  \uparrow$, and $  x \lhd y \rhd z =x \lhd \frac{y}{y}\rhd z.$}
 \[x \lhd y \rhd z = \frac{y}{y}\cdot x + (1 -\frac{y}{y})\cdot z\]
Notwithstanding the fact that $P^\uparrow(-)$ 
corresponds best with what ordinary school mathematics has to say about division, properly formulating
 its logic is far more involved than developing the logic needed to work with 
$P^0(-|-)$ or with $P^1(-|-)$ or $P^S(-|-)$.\footnote{%
The choice made below for using $P^0(-|-)$ rather than $P^1(-|-)$  or $P^S(-|-)$ is merely a matter of taste.} 
Using the conditional operator the definitions of $P^1(- | -)$ and $P^S(-|-)$ can be made more illuminating:
\[P^1(x \midd y)= 	\frac{P(x \wedge y)}{P(y)}  \lhd P(y) \rhd 1 ~~\textrm{and}~~
P^S(x \midd y)= 	  \frac{P(x \wedge y)}{P(y)} \lhd P(y) \rhd P(x).\]

The literature on conditional probabilities taking probability zero
for the condition into account is quite complex. Popper functions, nonstandard probabilities, 
Renyi's conditional probability, and
De Finetti's coherent conditional probability come into play. 
Working with $P^0(-)$ excludes some of these options for dealing with conditional 
probability functions, but choosing to work with an involutive meadow 
 does not by itself introduce such kind of bias. On the contrary, working in an involutive meadow 
 allows to proceed with the formalization of each of the 
mentioned options (and more) for the definitions of conditional probability functions.

\subsection{Relevance of meadows for the work in this paper}
The potential advantages of the use of meadows as the domain of values for a probability 
function  in the context of probabilistic reasoning are surveyed in the following items.
\begin{enumerate}
\item All relevant equations and conditional equations can be written such as to be valid in a two valued classical logic as understood 
for all substitutions of values for variables.\footnote{%
No attempt is made in the paper to work out these matters in full detail. In many cases instead of an assumption
that say $t \neq 0$, for a closed term $t$, it is assumed or derived that $t/t = 1$, which given the theory of meadows 
amounts to the same. The claimed advantage is that working with meadows allows to achieve 100\% 
semantical precision in these matters in principle. Compared with a  conventional style of mathematical writing 
working with meadows does not imply or induce any additional 
commitment to a formalistic and possibly overly detailed approach.}
\item Proofs of equations and of conditional equations can be given relative to the equational proof system $\BA+\Md+\Sign+\PFP$, 
often used in combination with some additional operator definitions. No separate
import of a theory of real numbers or of set theory is required.
\item 
\label{nonclassical} The following semantic problem that permeates conventional school mathematics is avoided. 
It is not uncommon to insists that in the world of rational numbers the following assertion $\Phi(x)$ holds in general. $\Phi(x)$ asserts
\[\displaystyle x \neq 0 \to x/x = 1\]
In other words $\forall x. \Phi(x)$ is considered  a valid assertion.
The idea is that difficulties regarding division by zero are excluded by the condition. 
Many presentations of Bayes' theorem implicitly presuppose this assumption. 
However, classical logic does not work that way. For the condition to play a role, it is plausible that $0$ can be substituted for $x$. 
And in order for $\forall x. \Phi(x)$ to be true it is required that $\Phi(0)$ is true, just as  all 
other substitution instances $\Phi(q)$ substituting a rational number $q$ for $x$ must hold. 
Truth of $\forall x. \Phi(x)$ implies the turth of:
\[\displaystyle \ 0 \neq 0 \to 0/0 = 1\]
In classical logic the validity of the proposition $\displaystyle \ 0 \neq 0 \to 0/0 = 1$ results from a bottom up definition 
of truth. Assuming a classical two-valued logic $0/0= 1$ must be either true or false. But conventional mathematics 
is reluctant to commit to either option. It appears that when dealing with fractions 
conventional mathematics deviates from classical first order logic.
\item Three-valued logics provide a solution to the dilemma mentioned in item~\ref{nonclassical}, but the corresponding
proof systems become harder to grasp  than probability theory proper. 
Close to conventional mathematical intuition is to work with partial functions and to formalize arithmetic using a logic of partial functions. 
Designing logics of partial functions can be done in many ways, however, while providing no obvious preferred options.

\item An alternative perspective on this matter is to maintain a temporal interpretation of implication,
thereby turning propositional calculus into a so-called short circuit logic. The idea is that if a condition $\phi$ of an implication $\phi \to \psi$  fails
the conclusion $\psi$ is left unevaluated and even its syntax or its well-definedness need not be checked.
Short-circuit logics have been worked out in ample detail in 
 Bergstra, Ponse \& Staudt~\cite{BergstraPS2013}. These details, however, are prohibitively complex for application in probability theory.
 
\item When taking numbers from a meadow the use of a conditional probability 
 $\displaystyle P(x \wedge y)/P(y)$ is always permitted and does not have the  side effect of introducing or implicitly presupposing 
 the assumption that $P(y)$ is non-zero.

An example may clarify the relevance of fact. In colloquial language and within informal 
mathematics an agent $A$ may ask an agent  $B$ for $B$'s guess of the conditional probability that some agent $C$ has sold a
particular object $F$ under the condition that $C$ has stolen that same object. 

Now, when working in a conventional and informal style,  by merely speaking  of this conditional probability,
 $A$ already implicitly states 
(requires, assumes) that the probability  that $C$ has stolen the object $F$ is non-zero. 
Having a meadow in mind $A$ is not committed to such an assumption.

\item The assertion that $P_A(H) \neq 0$ may be expressed as $P_A(H)/P_A(H) =1$. However, 
there is still a problem with this assertion because the constraint  $P_A(H)/P-A(H) =1$ cannot be conveniently expressed in the 
language of probability functions with precise beliefs. From the perspective  of subjective probability with beliefs 
represented by precise probability functions the assertion that $P_A(\texttt{C-has-stolen-F}) > 0$ 
is  informative about $A$'s belief, while at the same time it is uninformative as an 
exclamation by $A$ about their own state of belief. Indeed,  
at any time $A$ should be able, by definition of the concept of subjective probability, to encode
their level of uncertainty concerning ``$\texttt{C has stolen F}$''
in a probability function with precise values, e.g. by stating that $P_A(\texttt{C-has-stolen-F})= 0$ 
or by stating that  $P_A(\texttt{C-has-stolen-F})= 10^{-10}$. 
But questions remain: can it be the case that for some rational $r$,
$A$'s belief is $P_A(\texttt{C-has-stolen-F})= r$,
while representing $r$ precisely as a fraction is unfeasible for $A$ because it requires an extremely long expression? 
If so, can it be the case that $A$ cannot do better than to communicate an interval $P_A(\texttt{C-has-stolen-F})> a$ 
with rational $a$ chosen in such a way that denoting it is feasible?
\end{enumerate}

\subsection{Additional equations and conditional equations}
\label{AdditionalEandCE}
The completeness result of 
Bergstra \& Ponse ~\cite{BergstraP2016} allows to use all conditional equations which
are true in the meadow of real numbers. These are derivable from the finite equational theory 
$\BA+\Md+\Sign+\PFP$.
 
It should be noted that although $\BA+\Md+\Sign+\PFP$ is sound and 
complete for the case of real number valued probability functions, the same axiom system is merely sound but not complete for
rational number valued probability functions. For instance $(x^2-1)/ (x^2-1) =1$ is true in the meadow of rationals but is not derivable from
$\BA+\Md+\Sign+\PFP$.

The question whether or not a finite and complete equational axiomatisation of the 
equational theory of the meadow of rational numbers exists (as mentioned e.g. in~\cite{BBP2015}) is still open and so is the 
corresponding question for signed meadows in the presence of probability function.
It is also open whether or not a 
computably enumerable complete axiomatisation of the equational theory of the meadow of rational numbers exists.

\subsection{Transformations of proposition spaces and belief functions}
\label{trafos}
The work in this paper is sensitive to the precise shape of proposition spaces and probability functions on 
proposition spaces. Rather than working out these matters by way of 
preparation to the sequel of the paper, in full detail, I will limit this
 presentation to the mentioning of scattered aspects, while giving definitions by way of 
 representative examples rather than in a more general notational form. 
 
 The proposition space of an agent, say $A$, will be denoted with $S_A$. If it is known that the proposition 
 space is generated by primitive propositions, say $H,L,M,N$, I will write $S_A=S_A(H,L,M,N)$, 
 if there are only two generators one has e.g.\ $S_A(L,M)$ or $S_A(H,L)$.
 
 A belief function  $P_A$ (supposedly encoding beliefs of agent  $A$ at some moment of time) 
 maps each proposition in a proposition space to a value in a number space, 
 for which the meadow of rational numbers will be chosen below. 
 More sophisticated work may call for the meadow of reals.

 A belief function is best thought of as a pair $(S_A,P_A)$, though  the domain $S_A$ will often be left implicit. 
 Transformations of such pairs are core of Bayesian inference. The 
 following transformations on belief functions will play a role in this paper.
\begin{description}
\item [Bayes conditioning (without proposition kinetics).] Let for example $S_A=S_A(L,M,N)$. 
Suppose $P_A(M) = p$ with $p >0$. Then $\widehat{P_A}$ is obtained by Bayes 
conditioning if it satisfies the following equation:
\[\widehat{P_A} = P_A^0(\bullet \midd M)\]
This abbreviates that for all $X \in S_A$, $\widehat{P_A}(X) = P_A^0(X \midd M)$. 
It follows that $\widehat{P_A}(M)=1$ and the proposition space is left unaffected.\footnote{%
Bayes conditioning comes under alternative names: Bayes conditioning, Bayes' conditioning, 
Bayes conditionalization, Bayes' conditionalization, Bayesian conditioning, 
Bayesian conditionalization. In this paper only Bayes conditioning and Bayesian conditioning is used.}

\item [Bayes conditioning with proposition kinetics.] Let once more $S_A=S_A(L,M,N)$. 
Suppose $P_A(M) = p$ with $p >0$. Then $(S_A(L,N),\widehat{P_A})$ is obtained by Bayes
conditioning if $\widehat{P_A}$ satisfies the following equation:
\[\widehat{P_A} = P_A^0(\bullet \midd M)\]

Bayes conditioning with proposition kinetics 
removes $M$ from $S_A$ with the effect that after Bayes conditioning with respect to $M$ the 
proposition space has been reduced to $S_A=S_A(L,N)$. 

\item [Bayes conditioning on a non-primitive proposition.] Let, again by
way of example, the proposition space of $A$ have three generators: $S_A=S_A(L,M,N)$. 
Suppose $\Phi$ is a closed propositional sentence making use of primitives $L,N,$ and $M$. 
Suppose $P_A(\Phi) = p$ with $p >0$. Then $\widehat{P_A}$ is obtained by Bayes 
conditioning on $\Phi$  if it satisfies the following equation:
\[\widehat{P_A} = P_A^0(\bullet \midd\Phi)\]
When conditioning on a non-primitive proposition kinetics does not apply, i.e.\ the 
proposition space is left as it was.
\item [Jeffrey conditioning.] Let for example $S_A=S_A(L,M,N)$. Suppose $P_A(M) >0$. 
Then $\widehat{P_A}$ is obtained by Jeffrey 
conditioning if for some $p \in [0,1]$ it satisfies the following equation:
\[\widehat{P_A} = p \cdot P_A^0(\bullet \midd M) +  (1-p) \cdot P_A^0(\bullet \midd \neg M)\]
Jeffrey conditioning involves no proposition kinetics. Bayesian conditioning without proposition kinetics may 
be understood as the version of Jeffrey
 conditioning  with $p=1$.\footnote{%
Jeffrey conditioning has finite as well as infinitary versions. According to Diaconis \& Zabell~\cite{DiaconisZ1982} 
only its infinitary versions are stronger than any Bayesian rules.}

\item [Proposition space reduction.] Consider $S_A=S_A(L,M,N)$, one may wish to 
forget about say $M$. Proposition kinetics now leads to a reduced proposition space $S_A(L,N)$ 
in which only the propositions generated by $L$ and $N$ are left. 

Proposition space reduction constitutes the simplest form of proposition kinetics.
\item [Parametrized proposition space expansion.] Let $S_A=S_A(H)$. 
One may wish to expand $S_A$ to a proposition space by introducing $M$ to it in such a 
manner that a subsequent reduct brings one back in $S_A$.

$P_A(H)$ is left unchanged but $P_A(H \wedge M)$ and $P_A(\neg H \wedge M)$ must be fixed with definite values.
A specification of the new probability function, say $Q_A$, with domain $S(H,M)$, is thus: 
$Q_A(H) = P_A(H), Q_A(H \wedge M)= q_1, Q_A(\neg H \wedge M)= q_2$ with $q_1$ and $q_2$ 
appropriate rational number expressions. If one intends to extend $S_A=S_A(L,M)$ to $S_A=S_A(L,M,N)$ 
four additional values for the probability functions are needed and so on.

\item [Symmetric proposition space expansion.] Let $S_A=S_A(N,H)$. One may wish to expand $S_A$ to a 
proposition space by introducing $M$ to it in such a manner that a subsequent reduct brings one back in 
$S_A$ but one may not wish to guess any parameters. Now it suffices to assert that for each 
closed propositional expression $\Phi$ over the propositional primitives $N$ and $H$,  
$Q_A(M \wedge \Phi)= Q_A(M \wedge \neg \Phi)$, 
in other words all parameters are chosen with value $\frac{1}{2}$.

\item [Base rate inclusion.] This is a special case of parametrized proposition space expansion, and a 
generalization of symmetric proposition space expansion. 
Let $p$ be a closed value expression with $p>0$, 
and assume that $BR_h$ is a new proposition name. $BR_h$ is introduced in order to include the 
base rate $p$ (for some relevant type of event, named $h$) in the probability function. 

The probability function is extended as follows: $Q_A(\mathit{BR}_h \wedge \Phi)= p  \cdot Q_A(\Phi)$, 
for all sentences $\Phi$ not involving $\mathit{BR}_h$.

\item[Single likelihood Adams conditioning.]
Let $0< l \leq 1$ be  a rational number, (given by a closed expression for it). 
Assume that $H$ and $E$ are among the generators of $S_A$. Single likelihood 
Adams conditioning leaves the proposition space unchanged and 
transforms the probability function $P_A$ to $Q_l$ (leaving out the subscript $A$ for ease of notation).
\begin{equation*} Q_l =  P_A(H \wedge E \wedge \bullet) \cdot \frac{l}{P_A^0(E \midd H)}+
P_A(H \wedge \neg E  \wedge \bullet ) \cdot \frac{1-l}{P_A^0(\neg E \midd H)}+P_A( \neg H \wedge \bullet) \end{equation*}
\item[Double likelihood Adams conditioning.]
 Let $0< l, l^\prime \leq 1$ be  two rational numbers, (each given by a closed meadow expressions). 
 Assume that $H$ and $E$ are among the generators of $S_A$. 
 Double likelihood Adams conditioning leaves the proposition space unchanged and 
transforms the probability function $P_A$ to $Q_{l,l^\prime}$.\\
$\displaystyle  Q_{l,l^\prime} =  P_A(H \wedge E \wedge \bullet) \cdot \frac{l}{P_A^0(E \midd H)}+
 P_A(H \wedge \neg E  \wedge \bullet ) \cdot \frac{1-l}{P_A^0(\neg E \midd H)}+\\ 
 \quad~~~~~~~~~~ P_A(\neg H \wedge E \wedge \bullet) \cdot \frac{l^\prime}{P_A^0(E \midd  \neg H)}+
P_A(\neg H \wedge \neg E  \wedge \bullet ) \cdot \frac{1-l^\prime}{P_A^0(\neg E \midd  \neg H)}  $
\end{description}
 
\subsection{A labeled transition system of credal states}
\label{ltscs}
A pair $(S_A,P_A)$ may be understood as the logical counterpart of an agent $A$'s state of beliefs.
As $A$ may have beliefs not captured as assertions in $S_A$, $(S_A,P_A)$  is often 
referred to as $A$'s partial beliefs or as
 $A$'s partial state of beliefs.

In other words $(S_A,P_A)$ contains (as elements of $S_A$) and quantifies (via $P_A$) only some of the agent's beliefs. 
During a reasoning process $(S_A,Q_A)$
plays the role of a credal state in a model of the kinetics (dynamics) of $A$'s credences. 
Two credal states $(S_A,P_A)$ and 
$(S_A,Q_A)$ are called compatible if the same propositions (or rather sentences) 
of $S_A$ have probability $0$ (and thus $1$) under $P_A$ as under $Q_A$.

Now the collection $CS_U$ is defined as consisting of all  credal states with a Boolean algebra that is 
generated by a finite subset $W$ of a countable set $U$  of propositional atoms, 
and with a probability function taking values in the meadow of rational numbers..

Each of the transformations as outlined above in Paragraph~\ref{trafos} may be viewed as a (conditional) rule which generates 
 transitions between credal states. Transitions are labeled by the rule involved plus the parameters that are used for a specific transition.
 
 Labels are derived from rules, and the label created from a  rule
contains information about the name of the transformation and possibly of parameters, 
while the transition itself is between the prior and posterior state of the transformation. 
We will use the following labels (with $M \in U$ and $p$ a rational number in $[0,1]$):
\begin{itemize}
\item 
Bayes conditioning without proposition kinetics has a single parameter, and a single condition. The label is
$[bc;E]$ and the condition is $[P_A(E) > 0]$. 

The transition according to this rule requires of the 
prior credal state $(S_A(W),P_A)$ (with $W\subseteq U$),  that 
$E \in W$ and that $P_A(E) =p$. So if $P_A(E) > 0$ there is a transition 
 $(S_A(W),P_A) \to (S_A(W),P_A^0(\bullet \mid E))$ with label $[bc;E]$. 
\item Bayes conditioning with proposition kinetics generates transitions with label  $[bcpk;E]$ under the condition $[P_A(E) > 0]$.
\item Bayes conditioning on a nonprimitive proposition has label  $[bcnp;H]$ and condition $[P_A(H) > 0]$. 
\item Jeffrey conditioning has label $[jc;M]$.
\item Proposition space reduction has label $[psr,M]$.
\item Parametrized proposition space expansion has label $[ppse; M,N_1,\ldots,N_n,p_1,\ldots,p_{2^n}]$. 
Here $N_1,\ldots,N_n$ is an enumeration without repetition of elements of $U$ serving as generators for the 
proposition space at hand, and 
$p_1,\ldots,p_{2^n}$ enumerates, in a predetermined order, 
the new probabilities $\widehat{P_A}(M \wedge \phi)$ for $\phi$ an expression 
in conjunctive normal form over the mentioned generators.
\item Symmetric proposition space expansion has label $[sppse;M]$. 
This is just parameterized proposition space expansion with all $p_i$ equal to $1/2$.
\item Base rate inclusion has label $[bri;M;p]$. This rule generalises symmetric proposition space expansion and specializes 
parametrised proposition space expansion.
\item Single likelihood Adams conditioning has label $[slac;E,H;l]$, with $l$ a rational number in $(0,1]$.
\item Double likelihood Adams conditioning has label $[dlac;E,H;l,l^\prime]$, with $l, l^\prime$  rational numbers in $(0,1]$.

\end{itemize}

$\mathrm{CS}_A(U)$ denotes the collection of finite credal states for agent $A$ with a
proposition space generated by elements of  $U$. $\mathrm{CS}^{lts}_A(U)$ is $\mathrm{CS}_A(U)$ 
equipped with the structure of a labeled transition system, using the labeling mentioned above. 
This definition is limited by the requirement that quantities are rational. 
These definitions have obvious counterparts in case real numbers are used.

It follows from an observation in Diaconis \& Zabell~\cite{DiaconisZ1982} that if  $\Psi= (S_A, P_A)$ and 
$\Psi^\prime =(S_A,Q_A)$ 
are compatible credal states  it is possible to make a path of transitions $\mathrm{CS}^{lts}_A(U)$ in  
from $\Psi$ to $\Psi^\prime$ in two steps,
a parameterized proposition space expansion step followed by a Bayesian conditioning step.

%
%
%
%

\section{Belief kinetics and likelihood ratio transfer}
Likelihood ratio transfer mediated reasoning (LRTMR)  refers to a spectrum of reasoning patterns
used at the receiving side of probabilistic information.\footnote{%
Logical aspects  of courtroom proceedings worth of formal scrutiny and involving probabilistic information arise  in various contexts. 
For instance the implicit proof rule for the probability of a conjunction as listed in 
Arguello~\cite{Arguello} seems to be wrong and the reasoning pattern discussed by Stephens in~\cite{Stephens2011} 
lies outside the patterns considered below.} I will use LRTMR as a container for abstract  formalizations of various patterns of 
Bayesian reasoning.

In order to emphasize the general nature of the protocols and methods for LRTMR, 
and in order to simplify the presentation of expressions and proofs, I will use $A$ instead of TOF and $B$ instead of MOE.

In this Section it is assumed that the proposition space of $A$ is left unchanged during the reasoning process. 
In other words, there is only belief kinetics but no proposition kinetics.\footnote{%
In the literature on subjective probability theory instead of belief kinetics the phrase belief dynamics is used
as an alternative,
and instead of proposition kinetics the phrase proposition dynamics occurs.}
The outline of LRTMR in Paragraph~\ref{outline} below will serve as a point of departure for some more technical work.

\subsection{Evidence transfer mediated reasoning and the Taxi Color Case}
\label{TCC}
The simplest Bayesian 
reasoning pattern involving $A$'s reaction to and way of processing of information obtained from $B$ 
occurs if $A$ maintains a proposition 
space $S_A(E,H)$ and a belief function $P_A$ defined on it for which $P_A(E) = p > 0$.
In these circumstances  
it may happen that $B$ sends to $A$ a message to the extent that according to $B$ the proposition
$E$ is valid, or in other words that $P_B(E) = 1$. 
$A$ trusts $B$ and intends to adopt $B$'s belief that $E$ is true. 
A can do so by updating its belief function. Thus $A$ reacts to the input from $B$ by
performing Bayesian conditioning on $E$, thereby revising its belief function to $\widehat{P_A} = P^0(\bullet \midd E)$.

An example of this reasoning pattern occurs in the so-called Taxi Color Case as specified in detail in
Schweizer~\cite{Schweizer2006}.\footnote{%
Schweizer~\cite{Schweizer2006} contains a detailed description of the taxi color case, together with a useful survey of 
precise terminology in German about forensic reasoning patterns involving likelihood transfer and 
Bayesian conditioning. I will decorate the case description with some additional details. In
Schweizer~\cite{Schweizer2013} the similar bus color scenario is mentioned in an exposition concerning the 
legal value of base rates.} 
In a town (here TCCC, Taxi Color Case City), in total 1000 taxis circulate, 
150 of which are green and 850 of which are blue. A witness $W$ stated 
that (s)he saw a defendant $D$ leave with a green taxi from 
a specific location, more specifically $W$ saw $D$  departing with the first taxi in the taxi queue in front of restaurant $R$ at 23.00 PM.

Simplifying the case in comparison to Schweizer's description it is assumed that 
$A$ maintains an estimated base rate of 80\% for the correctness of $W$'s testimony. According to $A$'s 
background knowledge it may be expected in general for a witness operating in the conditions of $W$ at the 
time of the reported event that the witness (not just the actual witness $W$ but rather some 
average of test candidates) will report the color of the taxi correctly with a 
probability of 80\%.\footnote{%
In Schweizer's description, in contrast,  
$B$ investigates the statement of the witness, including an investigation concerning $W$'s ability to correctly report
about the color of a taxi including for instance (my details) information regarding the position from where she 
claimed to have been standing at the alleged time of $D$'s departure by taxi, and 
taking into account the overall illumination of the scene.}

$A$ is supposed to work with a proposition space with two propositions: 
$H$ (the hypothesis proposition asserting that $D$ left with a green taxi from the mentioned place and at the mentioned time), 
and $E$ 
(the evidence proposition asserting that according to $W$'s testimony $D$ left with a green taxi). 
$A$ uses, lacking other data, the 
base rate on operational taxi's (irrespective of location and time) of 150/1000 to set $P_A(H) = 150/1000$,
and $A$ uses the base rate of 80\% valid reporting (for both colors) to set: 
$P_A^0(E \midd H) = 80\%$ and $P_A^0( \neg E \midd \neg H) = 80\%$ whence $P_A(E \midd \neg H) = 100 \%-80\%= 20\%$, so that 
$P_A(E \wedge H) = P_A^0(E \midd H) \cdot P_A(H) = 80/100 \cdot 150/1000= 12/100$ and 
$P_A(E \wedge \neg H) = P_A^0(E \midd \neg H) \cdot P_A(\neg H) = 20/100 \cdot 850/1000 = 17/100$.
It follows that $P_A(E) = P_A(E \wedge H) + P_A(E \wedge \neg H)= 12/100 + 17/100 = 29/100$. 
$(S_A(E,H),P_A)$ serves as a prior belief state (credal state) for $A$.

Now one assumes that $A$ obtains evidence from $B$ in the form of $B$'s assertion ($P_B(E) =1$) that 
$W$ made a testimony which may be 
faithfully rendered at the relevant level of abstraction as amounting to $E$, so that $A$ may now assume that $E$ is true.
$A$ intends to adopt ($P_A(E) =1$) in its belief state, which must therefore be modified to, say $\widehat{P_A} = \widehat{P_A}(\bullet)$
(the function which assigns $\widehat{P_A}(X)$ to proposition $X$).
Given the acquired additional information $A$ may revise the quantitative 
consequences of its prior adoption of base rates, by
applying Bayesian conditioning (without proposition kinetics) relative to $E$. As $P^0_A(E) > 0$ a
 transition with label $[bc;E]$ can take place
from $(S_A(E,H),P_A)$ to $(S_A(E,H),\widehat{P_A})$ with $\widehat{P_A} = P^0(\bullet \midd E)$ (the
probability  function which assigns probability $P^0(\bullet \midd E)$ to proposition $X$). 
Now $A$ may calculate (rather than guess by introspection) its new belief in $H$:
\[\widehat{P_A} (H) = P_A^0(H \midd E) 
= \frac{P_A^0(E \midd H) \cdot P_A(H)}{P_A(E)}= \frac{80/100 \cdot 150/100}{29/100}
= \frac{12/100}{29/100}=12/29\]
The new  belief of $A$ in $H$ significantly exceeds the base rate of 150/1000 which was $A$'s prior belief in $H$. 
$A$ becoming aware of its new belief in $H$ concludes its episode of Bayesian reasoning as exemplified in this case.

\paragraph{The role of likelihoods and likelihood ratio.}
The example, as presented here and in contrast with Schweizer's presentation, 
 makes no use of the transfer of a single likelihood or  of a pair or a quotient of 
likelihoods by $B$. However, $A$ encounters two likelihoods
$$l = P_A^0(E \midd H) = 80/100 \quad \mathrm{and} \quad l^\prime =P_A^0(E \midd \neg H) = 20/100=1/5$$
and, implicitly the likelihood ratio $r = l/l^\prime = 4$.
 
The likelihoods $l$ and $l^\prime$ provide $A$  with information about the expected reliability of a witness in the case at hand. 
It is plausible that $B$ has obtained a value for $l$ and $l^\prime$ from elsewhere.  
Remarkably, as it turns out $r$ provides 
about  as much useful information for $A$ as the likelihood pair $(l,l^\prime)$

\paragraph{An example of likelihood ratio transfer.}
$A$ may take  $r=P_A^0(E \midd H)/P_A^0(E \midd \neg H)$ for a definition of $r$ and may
 investigate what value to assign to $r$ only after becoming involved in the case at hand. 
For instance $A$ may ask $B$ to investigate this figure and to report about it. Such an investigation may range from
measuring a base rate by averaging the performance of a sample consisting of several agents not themselves involved in the case, 
to a laboratory based examination of the relevant performance
of the actual witness $W$. In the taxi color case, $B$ may report to $A$ that 
$P_B^0(E \midd H) = 90\%$ and $P_A^0(E \midd \neg H) = 10\%$, i.e. a likelihood ratio of $9$. 
Then $A$ can adopt this information from $B$, and 
adapt its probability distribution to accommodate these figures and thereafter incorporate the information that $P_B(E)=1$ with the result that
(e.g. using Theorem~\ref{AdamsBayes}$\,$(5) below) a higher value for $\widehat{P_A}(H)$ is obtained:
\[\widehat{P_A}(H) =\frac{ 9 \cdot P_A(H)}{1 + (9-1) \cdot P_A(H)} = \frac{ 9 \cdot 15/100}{1 + (9-1) \cdot 15/100}= 135/220=27/44 > 12/29 \] 
In subsequent paragraphs a variety of interactions between $A$ and $B$  is considered where, in advance of providing 
evidence information regarding $E$,  $B$ determines a likelihood 
pair $P_B^0(E \midd H)$ and $P_B^0(E \midd \neg H)$ and subsequently conveys these either in separate steps, 
or in a single step as a pair or in a single step while merely transferring the ratio of both, in each case with the intent of
overruling the respective likelihoods or the likelihood ration which are given by $A$'s prior belief function. 

\paragraph{Further remarks on TCC}
The term ``likelihood'' is merely another word for conditional probability used in specific circumstances. 
In the example $A$ plays the role of TOF and $B$ plays the role of MOE.
The event of a single taxi departure in TCC provides 
a remarkably nice case study for theoretical work as it allows an amazing range of further details and significant complications, to mention:

(i) different methods for determining witness reliability,
the presence of multiple witnesses, witnesses with different reliability and 
with conflicting assertions, 
(ii) taking other colours, car model information, or 
number plate information into account, (iii) taking taxi management, scheduling, and monitoring into account, and
(iv) using improved base rate estimations.

\subsection{Outline of the LRTMR reasoning pattern}
\label{outline}
Often the term likelihood is used to denote a certain  conditional probability. We write $L_\alpha^0$ for a likelihood and 
$LR_\alpha^0$ for a particular ratio of likelihoods, commonly referred to as a likelihood ratio.
\[L^0_\alpha(X,Y)= P_\alpha ^0(X \midd Y) \quad \mathrm{and} \quad LR^0_\alpha(X,Y,\neg Y)= \frac{L^0_\alpha(X,Y)}{L^0_\alpha(X,\neg Y)}\]
It is now assumed that both $E$ and $H$ are among the generators of both proposition 
spaces $S_A$ and $S_B$. Further $A$ and $B$ have prior credal states $(S_A,P_A)$ and $(S_B,P_B)$. 
The reasoning protocol LRTMR involves the following steps:
\begin{enumerate}
\item It is checked by $B$  that $0<P_B(H)<1$ and $0 <P_B(E) <1$, otherwise $B$ raises an exception and the protocol aborts.\item $B$ 
determines the value $r$ of the likelihood ratio $\displaystyle \LR^0_B(E,H,\neg H) =\\ 
\frac{L^0_B(E,H)}{L^0_B(E,\neg H)} = \frac{P_B^0(E \midd H)}{P_B^0(E \midd \neg H)}$
with respect to its probability function $P_B$.
\item If $L^0_B(E,\neg H) = 0$ then  $B$ raises an exception and the protocol aborts.
\item \label{transfer}
$B$ communicates to $A$ the value $r$ and a description of $\LR^0_B(E,H,\neg H)$, that is a description of what
propositions $r$ is an \LR~of.
\item $B$ communicates its newly acquired information to $A$ that it now considers $P_B(E) = 1$, i.e. $E$ being true, 
to be an adequate representation of the state of affairs.
\item $A$ trusts $B$ to the extent that $A$ prefers those of $B$'s quantitative values that $B$ communicates 
during a run of the protocol over its own values for the same probabilities, likelihoods, and likelihood ratios.
\item $A$ takes all information into account and applies Bayesian conditioning to end up with its posterior belief 
function $\widehat{P_A}$ which satisfies:
\begin{equation}\widehat{P_A}(H) =\frac{ r \cdot P_A(H)}{1 + (r-1) \cdot P_A(H)}\end{equation}
\end{enumerate}
$A$ becoming aware of it having updated its beliefs in accordance with the ``Bayesian''  paradigm concludes the 
protocol.\footnote{%
The protocol is normative to the extent that $A$ is supposed to follow the calculation of 
$\widehat{P_A}(H)$ rather than some subjective guess on how to update its beliefs. Thus, on the one hand $A$ is 
required to start out with a subjectively determined prior credal state, but $A$ is not entitled to enact  updates
subjectively, quite on the contrary, $A$'s updating conventions are prescribed by the protocol.}

The equation that specifies the posterior belief on $H$ is equivalent in probability calculus to the more familiar odds form of Bayes' Theorem:
\[\frac{\widehat{P_A}(H)}{\widehat{P_A}(\neg H)}= r \cdot \frac{P_A(H)}{P_A(\neg H)} \]
The proof consists of a trivial computation when using equations and conditional equations such as discussed in Paragraph~\ref{AdditionalEandCE}.

The description of LRTMR is a drastic abstraction used for the purposes of the paper and diverse aspects are left 
unspecified including:
(i) has  an invitation to $B$ occurred for it to play a role in the protocol,  
(ii) how is an abort of the revision process performed when necessary, (iii) are any assumptions about the absence of 
background knowledge required for either $A$ or for $B$, (iv) making sure that checking 
various conditions does not  involve information transfer between $A$ 
and $B$ which stands in the way of the properly performing the conditioning operations? 

\subsection{Belief kinetics  I:  single likelihood Adams conditioning and representation independence}
\label{Adams:1}
I will first consider an adaptation of the protocol named SLTMR for successive (or separate) likelihood transfer mediated reasoning.
SLTMR results from LRTMR by modifying step~\ref{transfer} as follows:

First $B$ determines $l$ and $l^\prime$ such that $l = L^0_B(E,H)$, $l^\prime = L^0_B(E,\neg H)$, and $\displaystyle r = l/l^\prime$.\footnote{%
It is assumed that $l$ and $l^\prime$ are known as closed expressions with non-zero and non-negative value not in excess of 1 
for the meadow of rational numbers. This assumption is implicitly used many times below in order to
be able to apply $t/t= 1$ for various terms t. The same use is made of non-zero prior odds $P_A(H)$ and $P_A(E)$ 
which must as well be known in terms of such expression so as to guarantee $P_A(H)/P_A(H)= 1$ and
$P_A(E)/P_A(E)= 1$.}
Then $B$ successively communicates first $l$ and then $l^\prime$ to $A$, in both cases in addition to information 
concerning what sentences these values are likelihoods of.

In order to process the incoming information concerning $l$ and $l^\prime$, $A$ first applies the following transformation, 
thereby obtaining an intermediate (precise) belief function $Q_l$:
\begin{equation} Q_l =  P_A(H \wedge E \wedge \bullet) \cdot \frac{l}{P_A^0(E \midd H)}+
P_A(H \wedge \neg E  \wedge \bullet ) \cdot \frac{1-l}{P_A^0(\neg E \midd H)}+P_A( \neg H \wedge \bullet) \end{equation}
Following the exposition of Bradley~\cite{Bradley2005} this is the Adams transformation corresponding to 
an intended update of likelihood $L^0_A(E,H) $ to value $l$.

Next $A$ applies Adams conditioning to $Q_l$ in order to 
update its likelihood $L_A(E,\neg H) $ to value $l^\prime$, 
thus obtaining a second intermediate belief function $R_{l.l^\prime}$:
\begin{equation} R_{l,l^\prime} =  Q_l(\neg H \wedge E \wedge \bullet) \cdot \frac{l^\prime}{Q_l^0(E \midd  \neg H)}+
Q_l(\neg H \wedge \neg E  \wedge \bullet ) \cdot \frac{1-l^\prime}{Q_l^0(\neg E \midd  \neg H)}+Q_l (H \wedge \bullet) \end{equation}

Finally $A$ applies Bayesian conditioning to ${R_{l.l^\prime}} $ with respect to $E$, thereby obtaining its posterior beliefs 
$\widehat{P_A}$:
\begin{equation}\widehat{P_A} = \widehat{R_{l,l^\prime}} = R_{l.l^\prime}^{\,0}(\bullet \midd E) \end{equation}
The following facts can be shown concerning this sequence of three conditioning steps:
\begin{theorem}\label{AdamsBayes} 
Given the assumptions and definitions mentioned above, 
in particular assuming 
$l/l= l^\prime/l^\prime = P_A(H)/P_A(H) = P_A(E)/P_A(E) = P_A(E \wedge H)/P_A(E \wedge H)=$\\
	$P_A(E \wedge \neg H)/P_A(E \wedge \neg H)=
	P_A(\neg E \wedge H)/P_A(\neg E \wedge H) = P_A(\neg E \wedge \neg H)/P_A(E \wedge \neg H)=1 $, 
	the following identities 
are true for $l,l^\prime, r,P_A,Q_l, R_{l,l^\prime}$, 
and $\widehat{R_{l,l^\prime}}$:
	\begin{enumerate}

	\item $Q_l^0(E \midd H) = l$
	\item $R_{l,l^\prime}^0(E \midd H) = l$
	\item $R_{l,l^\prime}^0(E \midd \neg H) = l^\prime$
	\item  \label{BCdoubleL}
	$\displaystyle \frac{R_{l,l^\prime}^{\,0}(E \midd H)}{R_{l,l^\prime}^{\,0}(E \midd \neg H)} = r$
	\item \label{valueonH} $\displaystyle \widehat{R_{l,l^\prime}}(H) =\frac{ r \cdot P_A(H)}{1 + (r-1) \cdot P_A(H)}$
	\item 
$\displaystyle R_{l,l^\prime}(X)=$
 $\displaystyle P_A(\neg H \wedge E \wedge X) \cdot \frac{l^\prime}{P_A^0(E \midd \neg H)} +
P_A(\neg H \wedge \neg E \wedge X) \cdot \frac{1-l^\prime}{P_A^0(\neg E \midd \neg H)} +\\ 
~\hspace{5.3mm}\quad\quad\quad P_A(H \wedge E \wedge X) \cdot \frac{l}{P_A^0(E \midd H)}+
	 P_A(H \wedge \neg E  \wedge X ) \cdot \frac{1-l}{P_A^0(\neg E \midd H)}
$
	\end{enumerate}
\end{theorem}

\begin{proof} The proof of Theorem~\ref{AdamsBayes} is a matter of calculation on the basis of the 
available equational axioms and definitions.
\begin{enumerate}
\item
\begin{enumerate}

\item
 $Q_l(H) $
\begin{align*}
&=(P_A(H \wedge E \wedge \bullet) \cdot \frac{l}{P_A^0(E \midd H)}+
	 P_A(H \wedge \neg E  \wedge \bullet ) \cdot \frac{1-l}{P_A^0(\neg E \midd H)}+P_A( \neg H \wedge \bullet))(H)\\
&=  P_A(H \wedge E ) \cdot \frac{l}{P_A^0(E \midd H)}+
	P_A(H \wedge \neg E) \cdot \frac{1-l}{P_A^0(\neg E \midd H)}+P_A( \neg H \wedge H) \\
&=
P_A(H \wedge E ) \cdot \frac{l \cdot P_A(H)}{P_A(E \wedge H)}+
P_A(H \wedge \neg E) \cdot \frac{(1-l) \cdot P_A(H)}{P_A(\neg E \wedge H)}\\
&=l \cdot P_A(H)+(1-l) \cdot P_A(H)\\
&=P_A(H)
\end{align*}

\item $\displaystyle Q_l(E \wedge H)$
\begin{align*}
&=(P_A(H \wedge E \wedge \bullet) \cdot \frac{l}{P_A^0(E \midd H)}+
	 P_A(H \wedge \neg E  \wedge \bullet ) \cdot \frac{1-l}{P_A^0(\neg E \midd H)}+P_A( \neg H \wedge \bullet))(H \wedge E)\\
&=P_A(H \wedge E \wedge H \wedge E) \cdot \frac{l}{P_A^0(E \midd H)}+
P_A(H \wedge \neg E \wedge H \wedge E) \cdot \frac{1-l}{P_A^0(\neg E \midd H)} +\\
&\quad \hspace{1,1mm}P_A( \neg H \wedge H \wedge E) \\
&=P_A(H \wedge E) \cdot \frac{l}{P_A^0(E \midd H)}\\
&=l \cdot P_A(H)
\end{align*}

\item $\displaystyle Q_l^0(E \midd H) = \frac{Q_l(E \wedge H)}{Q_l(H)}=\frac{l \cdot P_A(H)}{P_A(H)} = l\\$
\end{enumerate}
\item
$\displaystyle R^0_{l,l^\prime}(E \midd H) = \frac{R_{l,l^\prime}(E \wedge H)}{R_{l,l^\prime}(H)}=
	\frac{Q_l(H\wedge E \wedge H)}{Q_l(H \wedge H)} = \frac{Q_l( E \wedge H)}{Q_l( H)} = Q^0(E \midd H) = l$
\item
\begin{enumerate}
\item $\displaystyle R_{l,l^\prime}(\neg H)$
	\begin{align*}
	&=(Q_l(\neg H \wedge E \wedge \bullet) \cdot \frac{l^\prime}{Q_l^0(E \midd  \neg H)}+
	Q_l(\neg H \wedge \neg E  \wedge \bullet ) \cdot \frac{1-l^\prime}{Q_l^0(\neg E \midd  \neg H)}+Q_l (H \wedge 	
	\bullet))(\neg H)\\
	&=Q_l(\neg H \wedge E \wedge \neg H) \cdot 
	\frac{l^\prime}{Q_l^0(E \midd  \neg H)}+
	Q_l(\neg H \wedge \neg E  \wedge \neg H ) \cdot \frac{1-l^\prime}{Q_l^0(\neg E \midd  \neg H)}+
	Q_l (H \wedge \neg H) \\
	&= Q_l(E \wedge \neg H) \cdot \frac{l^\prime}{Q_l^0(E \midd  \neg H)}+
	Q_l(\neg E  \wedge \neg H ) \cdot \frac{1-l^\prime}{Q_l^0(\neg E \midd  \neg H)}\\
	&= l^\prime \cdot Q_l(\neg H) + (1-l^	\prime) \cdot Q_l(\neg H) \\
	&= Q_l(\neg H)
	\end{align*}
	\item $R_{l,l^\prime}(E \wedge  \neg H)$
	\begin{align*}
	&= (Q_l(\neg H \wedge E \wedge \bullet) \cdot \frac{l^\prime}{Q_l^0(E \midd  \neg H)}+
	Q_l(\neg H \wedge \neg E  \wedge \bullet ) \cdot \frac{1-l^\prime}{Q_l^0(\neg E \midd  \neg H)}+Q_l (H \wedge 	
	\bullet))(E \wedge \neg H)\\
	&= Q_l(\neg H \wedge  E  \wedge \neg H ) \cdot \frac{l^\prime}{Q_l^0( E \midd  \neg H)}\\
	&= Q_l( E  \wedge \neg H ) \cdot \frac{l^\prime}{Q_l^0( E \midd  \neg H)}\\
	&= Q_l( E  \wedge \neg H ) \cdot \frac{l^\prime \cdot Q_l(  \neg H)}{Q_l( E \wedge  \neg H)}\\
	&= l^\prime \cdot Q_l(  \neg H)
	\end{align*}
	
	\item $\displaystyle R^0_{l,l^\prime}(E \midd  \neg H) = 
	\frac{R_{l,l^\prime}(E \wedge \neg H)}{R_{l,l^\prime}(\neg H)}	= 
	\frac{l^\prime \cdot Q_l(  \neg H)}{Q_l(\neg H)} 	= l^\prime$
	\end{enumerate}
\item Using the preceding items and by definition of $r$.
\item \begin{enumerate}
	\item $Q_l(\neg H) = 1- Q_l(H) = 1 -P_A(H)$
	\item 
	$ R_{l,l^\prime}(E)$
	\begin{align*}
	&= (Q_l(\neg H \wedge E \wedge \bullet) \cdot \frac{l^\prime}{Q_l^0(E \midd  \neg H)}+
	Q_l(\neg H \wedge \neg E  \wedge \bullet ) \cdot \frac{1-l^\prime}{Q_l^0(\neg E \midd  \neg H)}+\\
	&  \quad \hspace{1mm} Q_l (H \wedge 	
	\bullet))(E)\\
	&=Q_l(\neg H \wedge E) \cdot \frac{l^\prime}{Q_l^0(E \midd  \neg H)}
	+ Q_l (H \wedge E)\\
	&= l^\prime \cdot Q_l(\neg H) + l \cdot P_A(H) \\
	&= l^\prime \cdot (1-P_A(H)) + l \cdot P_A(H) 
	\end{align*}
	
	\item $ \displaystyle \widehat{R_{l,l^\prime}}(H)= R^0_{l,l^\prime}(H \midd E) = 
	\frac{R_{l,l^\prime}(E \wedge H)}{R_{l,l^\prime}(E)}=\frac{Q_{l}(E \wedge H)}{R_{l,l^\prime}(E)}$
	\begin{align*}
	&= 	\frac{l \cdot P_A(H)}{l^\prime \cdot (1-P_A(H)) + l \cdot P_A(H)}
	= \frac{l/l^\prime \cdot P_A(H)}{(1-P_A(H)) + l/l^\prime  \cdot P_A(H)}\\
	&= \frac{ r \cdot P_A(H)}{1 + (r-1) \cdot P_A(H)}
	\end{align*}
	\end{enumerate}

\item 	\begin{enumerate}
\item $\displaystyle Q_l^0(E \midd \neg H) = \frac{Q_l(E \wedge \neg H)}{Q_l(\neg H)}=
	\frac{P_A(E \wedge \neg H)}{1-P_A(H)} = P^0_A(E \midd \neg H)$

 \item
 $Q_l( \neg H) $
\begin{align*}
&=(P_A(H \wedge E \wedge \bullet) \cdot \frac{l}{P_A^0(E \midd H)}+
	 P_A(H \wedge \neg E  \wedge \bullet ) \cdot \frac{1-l}{P_A^0(\neg E \midd H)}+
	 P_A( \neg H \wedge \bullet))(\neg H)\\
&=  P_A(\neg H)
\end{align*}

\item $\displaystyle Q_l(\neg E \wedge \neg H)$
\begin{align*}
&=(P_A(H \wedge E \wedge \bullet) \cdot \frac{l}{P_A^0(E \midd H)}+
	 P_A(H \wedge \neg E  \wedge \bullet ) \cdot \frac{1-l}{P_A^0(\neg E \midd H)}+
	 P_A( \neg H \wedge \bullet))(E \wedge \neg H)\\
&=P_A(\neg E \wedge \neg H) 
\end{align*}

\item $\displaystyle Q_l^0(\neg E \midd \neg H) = \frac{Q_l(\neg E \wedge \neg H)}{Q_l(\neg H)}=
\frac{P_A(\neg E \wedge \neg H)}{P_A(\neg H)} = P_A^0(\neg E \midd \neg H)$

\item $R_{l,l^\prime}(X)$
\begin{align*}
&=  Q_l(\neg H \wedge E \wedge X) \cdot \frac{l^\prime}{P_A^0(E \midd \neg H)}+
Q_l(\neg H \wedge \neg E  \wedge X ) \cdot \frac{(1-l^\prime) }{P_A^0( \neg E \midd  \neg H)}+Q_l (H \wedge X)\\
&= P_A(\neg H \wedge E \wedge X) \cdot \frac{l^\prime}{P_A^0(E \midd \neg H)} +
P_A(\neg H \wedge \neg E \wedge X) \cdot \frac{1-l^\prime}{P_A^0(\neg E \midd \neg H)} \\
&~~~+ P_A(H \wedge E \wedge X) \cdot \frac{l}{P_A^0(E \midd H)}+
	 P_A(H \wedge \neg E  \wedge X ) \cdot \frac{1-l}{P_A^0(\neg E \midd H)}
\end{align*}

\end{enumerate}
\end{enumerate}
\end{proof}

The final credal state $\widehat{P_A}(H)$ for $H$ does not depend on the way $l$ and $l^\prime$ are 
chosen so that $r = l/l^\prime$. In other words the protocol is
 independent from the way $r$ is written as a fraction. 
This form of  independence will be referred to as the local representation independence of the SLTMR 
reasoning method. Global representation independence refers to the constraint that for each proposition $L$ in the proposition space of $S_A$
it is the case that $\widehat{P_A}(L)$  does not depend on the way $l$ and $l^\prime$ are 
chosen so that $r = l/l^\prime$.

A symmetry argument yields that performing the respective Adams conditioning steps in the other 
order leads to the same result.

However, if Bayesian conditioning is performed after the first Adams transformation, the result depends on the representation of the
likelihood ratio as a fraction. To see this one may derive from the definition of $Q_l(E)$ that
$Q_l(E) =l \cdot P_A(H)+P_A(\neg H \wedge E)$ so  that 
\[\displaystyle Q^0_l(H \midd E) = \frac{Q_l(H \wedge E)}{Q_l(E)} =  
P_A(H \wedge E) \cdot \frac{l}{P_A^0(E \midd H)}\cdot \frac{1}{Q_l(E)} = \frac{l \cdot P_A(H)}{l \cdot P_A(H)+P_A(\neg H \wedge E)}\]
Given that $P_A(\neg H \wedge E)>0$, the latter result depends on $l$, that is on the choice of $l$ and $l^\prime$ given $r$. 
Thus the first stage, obtained after a single
Adams transformation is not locally representation independent, and for that reason not globally representation independent either.

\subsection{Belief kinetics II: double likelihood Adams conditioning }
\label{AdamsBayes2}
In this paragraph the proporties of double likelihood  Adams conditioning are considered in detail. This 
conditioning mechanism fits best with likelihood ratio transfer as it  
 simultaneously incorporates two likelihoods $l$ and $l^\prime$. The likelihoods
$l$ and $l^\prime$ 
 may in turn have been obtained by choosing for a given likelihood ratio $r$ (which $A$ may have received from $B$) 
appropriate values such that $r = \frac{l}{l^\prime}$:\\
$\displaystyle Q_{l,l^\prime} =  P_A(H \wedge E \wedge \bullet) \cdot \frac{l}{P_A^0(E \midd H)}+
 P_A(H \wedge \neg E  \wedge \bullet ) \cdot \frac{1-l}{P_A^0(\neg E \midd H)}+\\ 
 \quad~~~~~~~~~~ P_A(\neg H \wedge E \wedge \bullet) \cdot \frac{l^\prime}{P_A^0(E \midd  \neg H)}+
P_A(\neg H \wedge \neg E  \wedge \bullet ) \cdot \frac{1-l^\prime}{P_A^0(\neg E \midd  \neg H)} $
\\~\\
Subsequent conditioning with respect to $E$ is given by:
\begin{equation}\widehat{Q_{l,l^\prime}} = Q_{l.l^\prime}^{\,0}(\bullet \midd E) \end{equation}
For double likelihood Adams conditioning we will prove a result similar to Theorem~\ref{AdamsBayes} but with fewer conditions, 
though at the cost of
a more involved statement of the equations in the Theorem.
\begin{theorem}\label{SimAdamsBayes} Given the assumptions and definitions mentioned above, 
and moreover assuming $l/l= l^\prime/l^\prime = P_A(H)/P_A(H) = P_A(E)/P_A(E) =1 $, the following equations and 
conditional equations are true for $l,l^\prime, r,P_A, Q_{l,l^\prime}$, 
and $\widehat{Q_{l,l^\prime}}$:
	\begin{enumerate}
	\item $\frac{P_A(H \wedge E)}{P_A(H \wedge E)} = \frac{P_A(H \wedge \neg E)}{P_A(H \wedge \neg E)} =1 \to Q_{l,l^\prime}^0(E \midd H) = l $
	\item $\frac{P_A(\neg H \wedge E)}{P_A(\neg H \wedge E)} = \frac{P_A(\neg H \wedge \neg E)}{P_A(\neg H \wedge \neg E)}=1  \to Q_{l,l^\prime}^0(E \midd \neg H) = l^\prime$

	\item \label{independence} $\displaystyle \widehat{Q_{l,l^\prime}}(X) = 
		 \frac{r \cdot P_A^0(X \midd H \wedge E) \cdot  P_A(H)+
 	  P_A^0(X \midd \neg H \wedge E) \cdot P_A(\neg H)}{ r  \cdot   \frac{P_A(H \wedge E)}{P_A(H \wedge E)}  \cdot P_A(H) +  \frac{P_A(\neg H \wedge E)} {P_A(\neg H \wedge E)} \cdot P_A( \neg H) }$
	\item
	\label{localValue}  $\displaystyle \widehat{Q_{l,l^\prime}}(H) =\frac{r \cdot  \frac{P_A(H \wedge E)}{P_A(H \wedge E)} \cdot  P_A(H)}{  r  \cdot   \frac{P_A(H \wedge E)}{P_A(H \wedge E)}  \cdot P_A(H) +  \frac{P_A(\neg H \wedge E)} {P_A(\neg H \wedge E)}\cdot P_A( \neg H)}$
	\end{enumerate}
\end{theorem}
\noindent From these facts the following conclusions  can be drawn:
\begin{itemize}
\item Performing single likelihood Adams conditioning for $L^0(E \midd H) = l$ and for $L^0(E \midd \neg H) = l^\prime$ 
under the conditions of Theorem~\ref{AdamsBayes}
in either order is equivalent to double likelihood Adams conditioning. 
\item  
Double likelihood Adams conditioning\footnote{%
Alternative names: simultaneous Adams conditioning, or likelihood pair Adams conditioning.} for likelihoods $l$ and 
$l^\prime$ transforms $P_A$ in to a probability function $P^\prime_A (= \widehat{Q_{l,l^\prime}})$ with the following properties (under reasonable conditions): (i)
the prior odds of $A$ for $H$ are protected in the sense that $P^\prime_A(H) = P_A(H)$,
(ii) for $P^\prime_A$: $L^0(E,H) =l$, and (iii) for $P^\prime_A$: $L^0(E,\neg H) =l^\prime$.
\item Double likelihood Adams conditioning is just one of many possible transformations that achieves the above requirements. 
Whether or not this particular transformation has a preferred status depends on the circumstances. Whether or not $A$, 
upon obtaining from $B$ an update of a likelihood pair or of a likelihood ratio, ought to be able to incorporate that new information by means of a canonical 
probability transformation of their belief state, and if so which transformation can play that role, is left undecided at the current level of abstraction.

\item Because $l$ and $l^\prime$ do not occur in the expression for $\displaystyle \widehat{Q_{l,l^\prime}}(X)$ 
double likelihood Adams conditioning  is 
globally representation independent. This feature of double likelihood Adams conditioning is in contrast with single likelihood 
Adams conditioning which is not globally representation independent.
\item In view of Theorem~\ref{AdamsBayes} the following answer to the question mentioned in the first lines of the 
paper is obtained: only by working with a 
likelihood ratio and by processing both ratios simultaneously global representation independence is obtained. 
\end{itemize}

\begin{proof} As double likelihood Adams conditioning may be considered the more important conditioning
transformation, in comparison to single likelihood Adams conditioning, the proofs have been worked out in full detail without making use of 
calculations for the single likelihood case.
\begin{enumerate}
\item \begin{enumerate} 
	\item $\displaystyle Q_{l,l^\prime}(E \wedge H)$
	\begin{align*}
	& = (	P_A(H \wedge E \wedge \bullet) \cdot \frac{l}{P_A^0(E \midd H)}+
 	P_A(H \wedge \neg E  \wedge \bullet ) \cdot \frac{1-l}{P_A^0(\neg E \midd H)}+\\
 	&\quad \hspace{1mm} P_A(\neg H \wedge E \wedge \bullet) \cdot \frac{l^\prime}{P_A^0(E \midd  \neg H)}+
	P_A(\neg H \wedge \neg E  \wedge \bullet ) \cdot \frac{1-l^\prime}{P_A^0(\neg E \midd  \neg H)})(E \wedge H)\\
	&=	P_A(H \wedge E) \cdot \frac{l}{P_A^0(E \midd H)} = l \cdot  \frac{P_A(H \wedge E)}{P_A(H \wedge E)} \cdot P_A(H)
	\end{align*}
	\item $\displaystyle Q_{l,l^\prime}(H)$
	\begin{align*}
	& = (P_A(H \wedge E \wedge \bullet) \cdot \frac{l}{P_A^0(E \midd H)}+
 	P_A(H \wedge \neg E  \wedge \bullet ) \cdot \frac{1-l}{P_A^0(\neg E \midd H)}+\\ 
 	&\quad \hspace{1mm} P_A(\neg H \wedge E \wedge \bullet) \cdot \frac{l^\prime}{P_A^0(E \midd  \neg H)}+
	P_A(\neg H \wedge \neg E  \wedge \bullet ) \cdot \frac{1-l^\prime}{P_A^0(\neg E \midd  \neg H)})(H) \\
	&=	 P_A(H \wedge E ) \cdot \frac{l}{P_A^0(E \midd H)}+
 	P_A(H \wedge \neg E  ) \cdot \frac{1-l}{P_A^0(\neg E \midd H)} \\
	&= \frac{P_A(H \wedge E)}{P_A(H \wedge E)}  \cdot l \cdot P_A(H) + 
		\frac{P_A(H \wedge \neg E)}{P_A(H \wedge \neg E)}  \cdot(1-l) \cdot P_A(H) =\\
	&=(l \cdot  \frac{P_A(H \wedge  E)}{P_A(H \wedge  E)}  + (1-l) \cdot \frac{P_A(H \wedge \neg E)}{P_A(H \wedge \neg E)}  ) \cdot P_A(H)
	\end{align*}
	\item 
	\label{itemc} $\displaystyle Q_{l,l^\prime}^0(E \midd H) = \frac{Q_{l,l^\prime}(E \wedge H)}{Q_{l,l^\prime}(H)}=
	\frac{l \cdot \frac{P_A(H \wedge E)}{P_A(H \wedge E)} }{l \cdot  \frac{P_A(H \wedge  E)}{P_A(H \wedge  E)}  + (1-l) \cdot \frac{P_A(H \wedge \neg E)}{P_A(H \wedge \neg E)}  }$
	\end{enumerate}
	
\item \begin{enumerate} 
	\item $\displaystyle Q_{l,l^\prime}(E \wedge \neg H)$
	\begin{align*}
	& = (	P_A(H \wedge E \wedge \bullet) \cdot \frac{l}{P_A^0(E \midd H)}+
 	P_A(H \wedge \neg E  \wedge \bullet ) \cdot \frac{1-l}{P_A^0(\neg E \midd H)}+\\
 	&\quad \hspace{1mm} P_A(\neg H \wedge E \wedge \bullet) \cdot \frac{l^\prime}{P_A^0(E \midd  \neg H)}+
	P_A(\neg H \wedge \neg E  \wedge \bullet ) \cdot \frac{1-l^\prime}{P_A^0(\neg E \midd  \neg H)})(E \wedge  \neg H)\\
	&=	P_A(\neg H \wedge E) \cdot \frac{l^\prime}{P_A^0(E \midd \neg H)} = l^\prime  \cdot \frac{P_A(E \wedge \neg H)}{P_A(E \wedge \neg H)} \cdot P_A(\neg H)
	\end{align*}
	\item $\displaystyle Q_{l,l^\prime}(\neg H)$
	\begin{align*}
	& = (P_A(H \wedge E \wedge \bullet) \cdot \frac{l}{P_A^0(E \midd H)}+
 	P_A(H \wedge \neg E  \wedge \bullet ) \cdot \frac{1-l}{P_A^0(\neg E \midd H)}+\\ 
 	&\quad \hspace{1mm} P_A(\neg H \wedge E \wedge \bullet) \cdot \frac{l^\prime}{P_A^0(E \midd  \neg H)}+
	P_A(\neg H \wedge \neg E  \wedge \bullet ) \cdot \frac{1-l^\prime}{P_A^0(\neg E \midd  \neg H)})(\neg H) \\
	&=	 P_A(\neg H \wedge E ) \cdot \frac{l^\prime}{P_A^0(E \midd \neg H)}+
 	P_A(\neg H \wedge \neg E  ) \cdot \frac{1-l^\prime}{P_A^0(\neg E \midd \neg H)} \\
	&= l^\prime \cdot \frac{P_A(\neg H \wedge E)}{P_A(\neg H \wedge E)} \cdot P_A(\neg H) + (1-l^\prime ) \cdot \frac{P_A(\neg H \wedge  \neg E)}{P_A(\neg H \wedge \neg E)} \cdot  P_A(\neg H) 
	\end{align*}
	Now the argument for part $2$ of the Theorem follows in the same way as the argument for part 1 in item~\ref{itemc} above.
	\item $\displaystyle Q_{l,l^\prime}^0(E \midd \neg H) = \frac{Q_{l,l^\prime}(E \wedge \neg H)}{Q_{l,l^\prime}(\neg H)}=
	\frac{l^\prime \cdot P_A( \neg H)}{P_A( \neg H)} = l^\prime$
	\end{enumerate}
	

\item \begin{enumerate}
	\item $\displaystyle Q_{l,l^\prime}(E) $
	\begin{align*}
	&= (P_A(H \wedge E \wedge \bullet) \cdot \frac{l}{P_A^0(E \midd H)}+
 	P_A(H \wedge \neg E  \wedge \bullet ) \cdot \frac{1-l}{P_A^0(\neg E \midd H)}+\\ 
 	&\quad \hspace{1mm} P_A(\neg H \wedge E \wedge \bullet) \cdot \frac{l^\prime}{P_A^0(E \midd  \neg H)}+
	P_A(\neg H \wedge \neg E  \wedge \bullet ) \cdot \frac{1-l^\prime}{P_A^0(\neg E \midd  \neg H)})(E)\\
	&=	P_A(H \wedge E) \cdot \frac{l}{P_A^0(E \midd H)}+
 	P_A(\neg H \wedge E) \cdot \frac{l^\prime}{P_A^0(E \midd  \neg H)}\\
	&= l \cdot \frac{P_A(H \wedge E)}{P_A(H \wedge E)}  \cdot P_A(H) + l^\prime \cdot \frac{P_A(\neg H \wedge E)}{P_A(\neg H \wedge E)}  \cdot P_A(\neg H)
	\end{align*}
	\item $\displaystyle Q_{l,l^\prime}(X\wedge E) $
	\begin{align*}
	&= (P_A(H \wedge E \wedge \bullet) \cdot \frac{l}{P_A^0(E \midd H)}+
 	P_A(H \wedge \neg E  \wedge \bullet ) \cdot \frac{1-l}{P_A^0(\neg E \midd H)}+\\ 
 	&\quad \hspace{1mm} P_A(\neg H \wedge E \wedge \bullet) \cdot \frac{l^\prime}{P_A^0(E \midd  \neg H)}+
	P_A(\neg H \wedge \neg E  \wedge \bullet ) \cdot \frac{1-l^\prime}{P_A^0(\neg E \midd  \neg H)})(X \wedge E)\\
	&=	P_A(X \wedge H \wedge E) \cdot \frac{l}{P_A^0(E \midd H)}+
 	P_A(X \wedge \neg H \wedge E) \cdot \frac{l^\prime}{P_A^0( E \midd  \neg H)}\\
	&=	P_A(X \wedge H \wedge E) \cdot \frac{l \cdot P_A(H)}{P_A(E \wedge H)}+
 	P_A(X \wedge \neg H \wedge E) \cdot \frac{l^\prime \cdot P_A(\neg H)}{P_A( E \wedge  \neg H)}\\
	&=	l \cdot P_A^0(X \midd H \wedge E) \cdot  P_A(H)+
 	 l^\prime \cdot P_A^0(X \midd \neg H \wedge E) \cdot P_A(\neg H)
	 \end{align*}
	\item $\displaystyle \widehat{Q_{l,l^\prime}}(X)\hspace{2,2mm}=\hspace {1mm}\frac{Q_{l,l^\prime}(X \wedge E)}{Q_{l,l^\prime}(E)}$
	\begin{align*}
	&= \frac{l \cdot P_A^0(X \midd H \wedge E) \cdot  P_A(H)+
 	 l^\prime \cdot P_A^0(X \midd \neg H \wedge E) \cdot P_A(\neg H)}{ l \cdot \frac{P_A(H \wedge E)}{P_A(H \wedge E)}  \cdot P_A(H) + l^\prime \cdot \frac{P_A(\neg H \wedge E)} {P_A(\neg H \wedge E)}  \cdot P_A( \neg H)}\\
	 &=  \frac{l^\prime \cdot (l /l^\prime \cdot P_A^0(X \midd H \wedge E) \cdot  P_A(H)+
 	  P_A^0(X \midd \neg H \wedge E) \cdot P_A(\neg H))}{l^\prime ( l/l^\prime  \cdot   \frac{P_A(H \wedge E)}{P_A(H \wedge E)}  \cdot P_A(H) +  \frac{P_A(\neg H \wedge E)} {P_A(\neg H \wedge E)}  \cdot P_A( \neg H) ) }\\
	  &=	  \frac{r \cdot P_A^0(X \midd H \wedge E) \cdot  P_A(H)+
 	  P_A^0(X \midd \neg H \wedge E) \cdot P_A(\neg H)}{ r  \cdot   \frac{P_A(H \wedge E)}{P_A(H \wedge E)}  \cdot P_A(H) +  \frac{P_A(\neg H \wedge E)} {P_A(\neg H \wedge E)} \cdot P_A( \neg H) }
	 \end{align*}
	\end{enumerate}
\item $\displaystyle \widehat{Q_{l,l^\prime}}(H) = \frac{r \cdot P_A^0( H \midd H \wedge E) \cdot  P_A(H)+
 	  P_A^0(H \midd \neg H \wedge E) \cdot P_A(\neg H)}{  r  \cdot   \frac{P_A(H \wedge E)}{P_A(H \wedge E)}  \cdot P_A(H) +  \frac{P_A(\neg H \wedge E)} {P_A(\neg H \wedge E)}\cdot P_A( \neg H)}$\\
	  \begin{align*}
	  &=  \frac{r \cdot  \frac{P_A(H \wedge E)}{P_A(H \wedge E)} \cdot  P_A(H)}{  r  \cdot   \frac{P_A(H \wedge E)}{P_A(H \wedge E)}  \cdot P_A(H) +  \frac{P_A(\neg H \wedge E)} {P_A(\neg H \wedge E)}\cdot P_A( \neg H)}
	  \end{align*}
\end{enumerate}
\end{proof}

\section{Proposition kinetics for LRTMR} 
In this Section it is 
assumed that initially $E$ is not yet included in the proposition space $S_A$ of $A$.  This assumption deviate from the assumptions
underlying single likelihood  Adams conditioning as well as double likelihood Adams conditioning.

 The following scenario indicates why it may be reasonable to assume that the evidence proposition is initially not known to $A$ 
 and for that reason not included in its proposition space $S_A$.
\begin{enumerate}
\item  $A$ may ask $B$ to provide evidence of relevance concerning proposition 
$H$ without having a particular and technically specific form of such evidence in mind; for instance $A$ may suggest 
$B$ to consider ``something with DNA'' instead of a more precise indication of what sort of technology is to be used.
\item A mere name $E$ for a proposition yet to be designed is agreed upon between $A$ and $B$ intended to be used
 for expressing what $B$ proposes that can be said about evidence for $H$ that is available to $B$.
\item  $B$ may subsequently propose to $A$ to make use of (to instantiate the template with) 
an evidence proposition ``of the type $T$'',  and $A$ may agree upon this plan with $B$. In other words it is greed that $E$ will have  type $T$.
\item $B$ communicates an abstract form $E^a$ of $E_{B}$ to $A$, which will serve as the public version of $A$'s evidence proposition, 
meant for the interactive reasoning in cooperation with $A$.
\item $A$ and $B$ agree to use $E$ as the name for $E^a$.
\item $A$ confirms that $E$ is sufficiently new given its background information. This means 
that the current $P_A$ has not come about by conditioning or constraining (the most prominent 
alternative revision mechanism in the imprecise case) on any proposition comparable to $E$. 
If no confirmation along these lines can be obtained by $A$  this thread of interaction with $B$ is aborted.
\item $A$ makes the plan to incorporate $E$ in its proposition space (proposition space family) using proposition kinetics:
\begin{itemize}
	\item the plan is to be carried out once relevant probabilistic information is made available to $A$ by $B$;
	\item $A$ need not become fully aware of the meaning of $E$ (i.e.\ of $E_B$) at any stage during the reasoning process;
	\item  $A$ must be able to embed enough information regarding 
	$E^a$ in its background knowledge base $K_B$ that in a forthcoming situation it may recognize a (high) 
	degree of similarity with the contents of a new proposition say $E^\prime$ which may be 
	proposed to $A$ by the same or another forensic expert (in the same, or another (?) case) 
	with the effect that $A$ must refuse subsequent Bayesian conditioning on $E^\prime$.\footnote{%
	This refusal is essential only after a reasoning step involving conditioning on $E$ has been performed.}
	\end{itemize}
\item  Therefore at this stage 
 $A$ is ready to receive information related to the original beliefs in $E$ (i.e. $E_B$) and $H$. 
 Preferably this is done by way of $B$ first sending to $A$ a likelihood ratio $r = \LR^0_B(E,H,\neg H)$.
$A$ will use this information to expand its proposition space with $E$, and (ii) to extend its (precise) belief function.

 Here it is assumed that  $B$ will report to $A$ about $E$  what it actually thinks (believes about) of $E_B$. 
 In some circumstances $B$ 
reports its past beliefs rather than current beliefs. Indeed $B$ may already have established that $P_B(E) = 1$ before
communicating a likelihood ratio $r = \LR^0_B(E,H,\neg H)$ to $A$. But once $P_B(E) = 1$,  unavoidably (for $B$) 
$LR_B(E,H,\neg H)=1$ just as well, a value not worth of being communicated.\footnote{%
This is a difficult point as $B$ is unlikely to agree with $A$ on the use of $E_B$, to begin with, if $B$ is subsequently 
unable to say anything relevant (though unknown to $A$) about its findings on $P_B(E)$. 
The fact that $B$ engages 
in the protocol at all 
provides statistically relevant information to $A$ which $A$ can, but should not, use for a subjective update of its priors.}

\item $A$ and $B$ may now proceed with the protocol as it has been
 specified above without proposition kinetics. 
\end{enumerate}

\subsection{Proposition kinetics in a  four element proposition space: global representation independence}
\label{PSkinOneGen}
It will be assumed that initially $H$ is in the proposition space of $A$ while $E$ is not. 
The simplest nontrivial proposition space $S_A$ has a single generator $H$ which is 
neither $\top$ nor $\bot$ so that $\top, \bot, H$, and $\neg H$ are the four elements of $S_A= S_A(H)$.
Initially it is assumed that $0< P_A(H) = p <1$.

It is then assumed that  $A$ receives from $B$ the information that $LR^0_B(E,H,\neg H) = r$  with $r > 0$.
Extending its proposition space $S_A$ with $E$
leads to a proposition space with two generators $H$ and $E$, and  
16 elements: 
\[|S_A(H,E)| = \{\top, \bot, H,\neg H, E, \neg E, H \wedge E, H \wedge \neg E, \neg H \wedge E, \neg H \wedge \neg E\}\] 
In order to specify a belief function $Q_A$ on this extende proposition space it satisfies to 
specify besides $Q_A(H)= p$ (inherited from $P_A$ the values $Q_A^0(E \midd H) = l$ and 
$Q_A^0(E \midd \neg H) = l^\prime$.

Now it is assumed that upon receiving the trusted information that $LR^0_B(E,H,\neg H) = r$, $A$ 
guesses values $l$ and $l^\prime$ for the underlying ratios (undisclosed by $B$) 
such that $r = \frac{l}{l^\prime}$. $A$ applies proposition kinetics by simultaneously 
extending $S_A$ to $S_A(H,E)$ and by specifying $Q_A$ so that $Q_A(H) = p$, $Q_A^0(E \midd H) = l$
and $Q_A^0(E \midd \neg H) = l^\prime$ (for the chosen values $l$ and $l^\prime$).

The next phase in LRTMR is that $A$ receives the information that $P_B(E) = 1$ 
(the evidence proposition is found to hold true by trusted agent $B$)  and $A$ processes this information 
by applying Bayesian conditioning to the evidence proposition $E$, 
thereby obtaining its posterior belief function $\widehat{Q_A} = Q_A^0(\bullet \midd E)$. 
The only probability worth evaluating is $\widehat{Q_A}(H)$:

\noindent $\displaystyle \widehat{Q_A}(H) = Q_A^0(\bullet \midd E)(H) = Q_A^0(H \midd E) =\frac{Q_A^0(E \midd H)\cdot Q_A(H)}{Q_A(E)}= $
\begin{align*}
&=\frac{Q_A^0(E \midd H)\cdot Q_A(H)}{Q_A^0(E \midd H) \cdot Q_A(H) +
 Q_A^0(E \midd \neg H) \cdot Q_A(\neg H)}\\
 &=\frac{l\cdot P_A(H)}{l \cdot P_A(H) + l^\prime \cdot P_A(\neg H)}
 = \frac{l/l^\prime\cdot P_A(H)}{l/l^\prime \cdot P_A(H) + P_A(\neg H)} \\
 &= \frac{r\cdot P_A(H)}{r \cdot P_A(H) + 1-P_A(H)} 
 = \frac{r\cdot P_A(H)}{1 + (r -1)\cdot P_A(H)}
 \end{align*}

Evaluating $\widehat{Q_A}(H)$ produces precisely the value as required in the outline description of 
LRTMR in Paragraph~\ref{outline} above. This fact can be understood as local representation independence
of LRTMR in the case of proposition kinetics with a prior proposition space generated by the hypothesis proposition.
Under the constraint of a single proposition generated proposition space  global representation independence is guaranteed, 
because it coincides with local representation independence.

\subsection{Proposition kinetics on an arbitrary proposition space: local representation independence}
\label{PSkinTwoGen}
The situation may be reconsidered in the case of a proposition space $S_A$ which is generated by two 
propositions $H$ and $L$, $L$ now playing the role of the second hypothesis proposition. 
Upon receiving from $B$ the information that $LR^0_B(E,H,\neg H) = r$  with $r > 0$, $A$ 
extends its proposition space from $S_A=S_A(H,L)$ to $S_A(H,L,E)$.
In order to have a belief function $Q_A$ on this space extending the prior belief function $P_A$ 
on the prior proposition space $S_A$, $A$ must choose the following likelihoods: 
$Q_A^0(E \midd   H \wedge L) = u$, $Q_A^0(E \midd  H \wedge  \neg L) = v$,
$Q_A^0(E \midd  \neg H \wedge L) = u^\prime$, and $Q_A^0(E \midd  \neg H \wedge  \neg L) = v^\prime$.
These values must be chosen in such a manner that $\LR_A(E,H,\neg H) = r$ will hold. Therefore it is required that 
(with respect to $Q_A$):

$\displaystyle r = \LR^0_A(E,H,\neg H)= \frac{L^0_A(E,H)}{L^0_A(E,\neg H)}=\frac{Q_A^0(E \midd H)}{Q_A^0(E \midd  \neg  H)}$
\begin{align*}
&=\frac{Q_A(E \wedge H)}{Q_A(H)} \cdot \frac{Q_A(\neg H)}{Q_A(E \wedge \neg H)}\\
&= \frac{Q_A(E \wedge H \wedge L )+ Q_A(E \wedge H \wedge \neg L )}{Q_A(H)} \cdot 
	\frac{Q_A(\neg H)}{Q_A(E \wedge \neg H \wedge L)+ Q_A(E \wedge \neg H \wedge \neg L)}\\
	&=\frac{u \cdot Q_A( H \wedge L )+ v \cdot Q_A( H \wedge \neg L )}{Q_A(H)} \cdot 
	\frac{ Q_A(\neg H)}{u^\prime \cdot Q_A(\neg H \wedge L)+ v^\prime \cdot Q_A(\neg H \wedge \neg L)}
\end{align*}

 The following probabilities can be calculated:\\
$Q_A(E)= u \cdot Q_A(H \wedge L) + v \cdot Q_A(H \wedge \neg L) + u^\prime \cdot Q_A(\neg H \wedge L) +
v^\prime \cdot Q_A(\neg H \wedge \neg L)$ and\\
$Q_A(E \wedge H) = u \cdot Q_A(H \wedge L) + v \cdot Q_A(H \wedge \neg L) $.

Upon receiving the information that (according to $B$) $P_B(E) = 1$, $A$ will perform 
Bayesian conditioning resulting in the posterior belief function $\widehat{Q_A} = Q_A^0(\bullet \midd E)$. 
Calculating $\widehat{Q_A}(H)$ produces:

\noindent $\displaystyle \widehat{Q_A}(H) = Q_A^0(\bullet \midd E)(H) = Q_A^0(H \midd E) = 
	\frac{Q_A(H \wedge E)}{Q_A(E)}$
	\begin{align*}
	&= \frac{u \cdot Q_A(H \wedge L) + 
		v \cdot Q_A(H \wedge \neg L)}{u \cdot Q_A(H \wedge L) + v \cdot Q_A(H \wedge \neg L) + 
		u^\prime \cdot Q_A(\neg H \wedge L) +
		v^\prime \cdot Q_A(\neg H \wedge \neg L)}\\
	&=
		\frac{u \cdot Q_A(H \wedge L) + v \cdot Q_A(H \wedge \neg L)}{u \cdot Q_A(H \wedge L) + 
			v \cdot Q_A(H \wedge \neg L) + 
			\frac{u \cdot Q_A( H \wedge L )+ v \cdot Q_A( H \wedge \neg L )}{r} \cdot 
			\frac{ Q_A(\neg H)}{Q_A(H)}}\\
	&= \frac{1}{1 + \frac{1}{r} \cdot \frac{ Q_A(\neg H)}{Q_A(H)}} 
	=\frac{r \cdot Q_A(H)}{1 + (r-1) \cdot Q_A(H) }\\
	&=	\frac{r \cdot P_A(H)}{1 + (r-1) \cdot P_A(H) }
	\end{align*}

It may be concluded that in the case of two generators for $S_A$ and an arbitrary 
guess for all new probabilities (upon introducing $E$ as a new generator) local
representation independence (that is independence with respect to  $H$) is obtained. Using a similar proof it can be shown that 
representation independence generalizes to an arbitrary number of generators for $S_A$. 

\subsection{Proposition kinetics for a 6 element proposition space: failure of global representation independence}
\label{NoglobalRD}
Global representation independence is a different matter as will be found by considering an example. Calculating $\widehat{Q_A}(L)$ produces:\\
\noindent $\displaystyle \widehat{Q_A}(L) = Q_A^0(\bullet \midd E)(L) = Q_A^0(L \midd E) = 
	\frac{Q_A(L \wedge E)}{Q_A(E)}= \\
	\frac{u \cdot Q_A(H \wedge L) + 
		u^\prime \cdot Q_A(\neg H \wedge L)}{u \cdot Q_A(H \wedge L) + v \cdot Q_A(H \wedge \neg L) + 
		u^\prime \cdot Q_A(\neg H \wedge L) +
		v^\prime \cdot Q_A(\neg H \wedge \neg L)}$

Now consider as an example the case that $Q_A( H \wedge L )= Q_A( H \wedge \neg L )=
	Q_A( \neg H \wedge L )=Q_A( \neg H \wedge \neg L ) = \frac{1}{4}$. Then the requirement on $u,v,u^\prime$ and $v^\prime$ simplifies to:
$\displaystyle r = \frac{u + v }{u^\prime + v^\prime}$, and $\widehat{Q_A}(L )$ simplifies as follows:  
$\displaystyle \widehat{Q_A}( L ) = \frac{u +u^\prime}{u + v + u^\prime + v^\prime}$. 
Now choosing $u^\prime = v^\prime = \frac{1}{3}$ and $r=\frac{3}{2}$ we find $u+ v=1$, for instance 
$u = \frac{3}{7}$ and $v= \frac{4}{7}$ or alternatively $u = \frac{4}{7}$ and $v= \frac{3}{7}$. In these two cases  
$\widehat{Q_A}( L )$ takes  different values. In the first case
$\displaystyle \widehat{Q_A}( L ) = \frac{u + u^\prime}{u + v + u^\prime + v^\prime}= 
\frac{3/7 + 1/3}{3/7 + 4/7 + 1/3 + 1/3}$ whereas in the second case: 
$\displaystyle \widehat{Q_A}( L ) =\frac{4/7 + 1/3}{3/7 + 4/7 + 1/3 + 1/3}$.
It follows that when $S_A$ has two or more generating propositions global representation independence fails.

\subsection{Bayesian conditioning followed by Jeffrey conditioning}
The process specified in Paragraph~\ref{PSkinTwoGen} has two disadvantages: failure of global representation independence, and 
pollution of the belief function with meaningless (guessed) values, due to the fact that the mere availability of a new
likelihood ratio leaves open many degrees of freedom. The second disadvantage, however, is immaterial because the 
first issue stands in the way of chaining the reasoning pattern with subsequent reasoning steps taking the obtained posterior belief function as a prior. 

The following process is plausible for $A$ upon it receiving the information that
trusted agent $B$'s beliefs imply $\LR^0_B(E,H,\neg H) = r$. First use the new information on 
$E$ in relation to $H$  to compute a new (revised) belief $\widehat{p}$ in $H$ according to the process as specified in Paragraph~\ref{PSkinOneGen}:
$$\widehat{p} = \frac{r \cdot P_A(H)}{1 + (r-1) \cdot P_A(H) }$$ This first step involves proposition kinetics. In
the second step, however, the proposition space of $A$ is not extended, instead merely a revision of the belief function is performed.

Recall that Jeffrey conditioning 
(with parameter $p$ on proposition $H$) 
works as follows
$$ \widehat{P_{p,H}} = p \cdot P^0(\bullet \midd H) + (1-p) \cdot P^0(\bullet \midd \neg H)$$
The revision of $P_A$ is found by the following application of Jeffrey conditioning:
$$ \widehat{P_A} = \widehat{P_{\widehat{p},H}} = 
\frac{r \cdot P_A(H)}{1 + (r-1) \cdot P_A(H) } \cdot P^0(\bullet \midd H) + (1-\frac{r \cdot P_A(H)}{1 + (r-1) \cdot P_A(H) }) \cdot P^0(\bullet \midd \neg H)$$
The two stage belief revision process just outlined produces a posterior belief state which may serve as a
prior belief state for a subsequent reasoning step.

\subsection{Adams followed by Bayes equals Jeffrey after Bayes}
\label{commute}
At first sight it seems that the case where $E$ is contained in the proposition space of $A$ is the more general case. 
Double likelihood Adams conditioning allows the independent processing, in terms of belief state revision  
by $A$, of an incoming likelihood ratio from B. If, however the subsequent phase of Bayes conditioning is included,
the path involving proposition space kinetics turns out to be the more general one. Below it will be 
proven that both revision processes commute. I will first establish the equivalence of both approaches by way of direct calculation. 
Subsequently a concise manner of formulating this and other equivalences with the help of conditioning 
combinators is provided.

Upon receiving a message $\LR^0_B(E,H,\neg H)= r$ A may first create a second proposition 
space generated by $H$ and $E$ and proceed as in~\ref{PSkinOneGen}, thereby producing a posterior probability 
$$\widehat{p} = \frac{r \cdot P_A(H)}{1 + (r-1) \cdot P_A(H) }$$ 
for $H$ after Bayes conditioning in the auxiliary proposition space. 

Now $S_A$ contains $E$ which allows Bayes conditioning on $E$, thus obtaining as an intermediate 
result $Q = P^0(\bullet \midd E)$. Subsequently Jeffrey conditioning with parameter $\widehat{p}$ and with
respect to $H$
may be applied to the intermediate probability function $Q$ thus obtaining $P_p$ as follows:
\[\widehat{P} = \widehat{p} \cdot Q^0(\bullet \midd H) +  (1-\widehat{p}) \cdot Q^0(\bullet \midd \neg H)\]
$\widehat{P}$  is a plausible result of performing the combination of receiving a likelihood ratio $r$ (for $E$ and $H$)
and a confirmation of $E$ from the prior belief $P_A$. We consider $\widehat{P}(x)$ for an arbitrary proposition $x$ in $S_A$:
\begin{align*}
\widehat{P}(x)&= (\widehat{p} \cdot Q^0(\bullet \midd H) +  (1-\widehat{p}) \cdot Q^0(\bullet \midd \neg H))(x) \\
&= (\widehat{p} \cdot \frac{Q(\bullet \wedge H)}{Q(H)} +  
(1-\widehat{p}) \cdot \frac{Q^0(\bullet \wedge \neg H)}{Q(\neg H)})(x)\\
&= (\widehat{p} \cdot \frac{P_A^0(\bullet \wedge H \midd E)}{P_A^0(H \midd E)} +  
(1-\widehat{p}) \cdot \frac{P^0_A(\bullet \wedge \neg H \midd E)}{P_A^0(\neg H \mid E)})(x)\\
&= (\widehat{p} \cdot \frac{P_A(\bullet \wedge H \wedge E) \cdot P_A(E)}{P_A(E) \cdot P_A(H \wedge E)} +  
(1-\widehat{p}) \cdot \frac{P_A(\bullet \wedge \neg H \wedge E) \cdot P_A(E)}{P_A(E) \cdot P_A(\neg H \wedge E)})(x)\\
&= (\widehat{p} \cdot \frac{P_A(\bullet \wedge H \wedge E)}{P_A(H \wedge E)} +  
(1-\widehat{p}) \cdot \frac{P_A(\bullet \wedge \neg H \wedge E) }{ P_A(\neg H \wedge E)})(x)\\
&= (\widehat{p} \cdot P_A^0(\bullet \midd H \wedge E) +  
(1-\widehat{p}) \cdot P_A^0(\bullet \midd \neg H \wedge E) )(x)\\
&= \widehat{p} \cdot P_A^0(x \midd H \wedge E) +  
(1-\widehat{p}) \cdot P_A^0(x \midd \neg H \wedge E)\\
&= \frac{r \cdot P_A(H) \cdot P_A^0(x \midd H \wedge E) }{1 + (r-1) \cdot P_A(H) }+ 
(1-\frac{r \cdot P_A(H)}{1 + (r-1) \cdot P_A(H) })\cdot P_A^0(x \midd \neg H \wedge E)\\
&= \frac{r \cdot P_A(H) \cdot P_A^0(x \midd H \wedge E) }{1 + (r-1) \cdot P_A(H) }+ 
(\frac{1- P_A(H)}{1 + (r-1) \cdot P_A(H) })\cdot P_A^0(x \midd \neg H \wedge E)\\
&= \frac{r \cdot P_A(H) \cdot P_A^0(x \midd H \wedge E) + P_A(\neg H) \cdot P_A^0(x \midd \neg H \wedge E)}{1 + (r-1) \cdot P_A(H) }
\end{align*}

It turns out that under the assumption that $P_A(E \wedge H) / P_A(E \wedge H)= 
P_A(E \wedge \neg H) / P_A(E \wedge \neg H)=P_A(\neg E \wedge H) / P_A(\neg E \wedge H)=
P_A(\neg E \wedge \neg H) / P_A(\neg E \wedge \neg H)=1$, $\widehat{P}(x)$ is identical to $\widehat{Q_{l,l^\prime}}(x)$ as found in 
Theorem~\ref{SimAdamsBayes}. This identity serves as a confirmation of the validity of each of 
the two  pathways which derive the same probability function on the same proposition space:
(i) double Adams conditioning followed by Bayes conditioning and,
(ii) the following sequence of steps:
\begin{enumerate}
\item starting a new proposition space with generator $H$, the probability being taken from $P_A$, 
\item proposition kinetics in the new proposition space adding $E$ to it such that the acquired likelihood ratio fits, 
\item Bayes conditioning with proposition kinetics in the auxiliary workspace, 
\item extracting the posterior probability $\widehat{p}$ of $H$ from the auxiliary proposition space, 
\item 
Bayes conditioning with proposition kinetics on the original (prior) proposition space of $A$, 
\item and finally Jeffrey conditioning (with respect to $\widehat{p}$ and $H$) on the result of the last step.
\end{enumerate}

\subsection{Proposition kinetics in advance of double likelihood Adams conditioning}
Starting with a belief state $(S_A,P_A)$ not containing proposition constant $E$, $A$ may first expand its proposition space 
using base rate inclusion $[bri;E;p]$ with $0<p<1$. Subsequently $A$ may apply a double likelihood Adams 
transformation with label $[dlac;E,H;l,l^\prime]$
($l$ and $ l^\prime$ both nonzero), and then $A$ may apply Bayesian conditioning without kinetics (label  $[bc;E]$).

It follows from the result in Paragraph~\ref{commute} that the resulting proposition space is independent from the choice of $p$. 
Thus proposition kinetics may be placed in advance of double likelihood Adams conditioning and subsequent Bayes conditioning.

\subsection{Limited merits of Adams conditioning}
\label{limitedM}
Double likelihood Adams conditioning with label $[dlac;E,H;l,l^\prime]$ with $l/l^\prime \neq 0$ transforms a belief state $(S_A,P_A)$ 
which involves a generator $E$ and a proposition $H$ into $(S_A,\widehat{P_A})$ in such a way that (i) the prior odds of $A$ regarding 
$H$  remain invariant i.e. $\widehat{P_A}(H) = P_A(H)$, and (ii) $LR^0_A(E,H, \neg H) =r$. In this update $E$ and $H$ are treated differently,
so that the subjectively given belief $H$ is protected, while $A$'s belief assertion $E$ which is supposed to be sensitive to scientifically grounded
observation is adapted: 
\[\frac{P_A(H \wedge E)}{P_A(H \wedge E)} =\frac{P_A(H \wedge \neg E)}{P_A(H \wedge \neg E)} = 1 \to \widehat{P_A}(E) = l.P_A(H) + l^\prime. (1-P_A(H))\] 
so that $\widehat{P_A}(E)$ differs from $P_A(E)$ for most 
choices of $l$ and $l^\prime$. The Adams transformation assigns different status to $A$'s priors regarding $H$ and $E$, regarding $A$'s prior for $E$ 
of lesser importance. Perhaps this aspect of Adams transformation disqualifies it from the perspective of subjective probability theory.

\section{Fallacies and single likelihood transfer based reasoning}
Instead of transferring a likelihood pair (simultaneously or in consecutive separate messages) or transferring a likelihood ratio, 
merely a single likelihood may be transferred by $B$ to $A$. Typically
in the literature on forensic reasoning the prosecution would expect MOE to provide such information to TOF. 

The resulting inferences
are often qualified as fallacies.\footnote{%
For a philosophical discussion of fallacies in the context of Bayesian reasoning I refer to Korb~\cite{Korb2004}} 
Transposing the conditional and the prosecutor's fallacy may be considered 
failed examples of attempts to design and use methods for single likelihood transfer mediated reasoning (SLTMR).
I will first focus on the so-called transposition of the conditional, a phrase attributed to Lindley by 
Fienberg \& Finkelstein~\cite{FienbergF1996}.

\subsection{What is wrong with transposition of the conditional?}
\label{toc}
Transposing the conditional is often portrayed as making the mistake that $P^0(H \midd E) = P^0(E \midd H)$ which
is admittedly easily refuted unless one of three ``unlikely'' conditions holds: $P(E) =0$, or $P(H)=0$, or $P(H) = P(E)$.

I consider this way of looking at what is wrong with transposition of the conditional (TOC) rather implausible.\footnote{%
In Paragraph~\ref{TCC} the implausibility of TOC as an inference mechanism in the absence of belief revision has 
already been argued in some detail.} 
In my view a more plausible way of looking at it takes into account the setting of new information and 
corresponding belief state revision. Instead of transposition of the conditional I will speak of transposition of the likelihood, 
which of course amounts to the same.
Now assuming that agent $A$ receives new information concerning the single likelihood $P_A^0(E \midd H)$, say 
$P_A^0(E \midd H)=l$. Is there a justification for $A$ to infer that after performing an 
appropriate belief revision resulting in $\widehat{P_A}$, the following identity may be correct: 
\[\widehat{P_A}^0(H \midd E) = l  (= P_A^0(E \midd H))\]
A plausible intuition runs as follows: if $P_A^0(E \midd H)$ is updated then $P_A^0(H \midd E)$ must be updated as well, perhaps 
not as much but in any case to some extent. And then it is suggested that the same modification applies, which is the fallacy at hand.

However, as will be shown below: the intuition is wrong: no update of a transposed likelihood can be inferred from the change of 
its transposed likelihood. Moreover there is nothing wrong in principle with receiving an update of a single likelihood and updating one's
subjective probabilities accordingly: that is the subject of Adams conditioning.
Following Theorem~\ref{AdamsBayes} it may be assumed  that $\widehat{P_A}$ is obtained from $P_A$ via single likelihood 
Adams conditioning. 
\[\widehat{P_A} = P_A(H \wedge E \wedge \bullet) \cdot \frac{l}{P_A^0(E \midd H)}+
P_A(H \wedge \neg E  \wedge \bullet ) \cdot \frac{1-l}{P_A^0(\neg E \midd H)}+P_A( \neg H \wedge \bullet)\]
The calculation of  $\widehat{P_A}^0(H \midd E) $ is covered in the following Theorem. The same conditions as for Theorem~\ref{AdamsBayes} are assumed.

\begin{theorem}
\label{TOCwrong} Under the assumption that single likelihood Adams conditioning is an 
appropriate belief transformation for $A$ in response to the reception of 
an update of a single likelihood the following holds.  

If $\widehat{P_A}$ is the posterior belief of $A$ upon acquiring knowledge that 
$P_A^0(E \midd H)=l$ then $\widehat{P_A}^0(H \midd E) = P_A^0(H \midd E).$
\end{theorem}
\begin{proof}
\noindent $\displaystyle \widehat{P_A}^0(H \midd E)= Q_l^0(H \midd E) = \frac{Q_l(H \wedge E)}{Q_l(E)}$
\begin{align*}
& = \frac{(P_A(H \wedge E  \wedge \bullet ) \cdot \frac{l}{P_A^0( E \midd H)})(H \wedge E)}
	{(P_A(H \wedge E \wedge \bullet) \cdot \frac{l}{P_A^0(E \midd H)}+P_A( \neg H \wedge \bullet))(E)}= \frac{P_A(H \wedge E  ) \cdot \frac{l}{P_A^0( E \midd H)}}
	{P_A(H \wedge E) \cdot \frac{l}{P_A^0(E \midd H)}+P_A( \neg H \wedge E)}=\\
& = \frac{l \cdot P_A(H)}{l \cdot P_A(H) + P_A( \neg H \wedge E)}= \frac{P_A^0(E \midd H) \cdot P_A(H)}{P_A^0(E \midd H) \cdot P_A(H) + P_A( \neg H \wedge E)}=\\
&= \frac{\frac{P_A(E \wedge H)}{P_A(H)} \cdot P_A(H)}{\frac{P_A(E \wedge H)}{P_A(H)} \cdot P_A(H) + P_A( \neg H \wedge E)}= \frac{P_A(E \wedge H)}{P_A(E \wedge H)+ P_A( \neg H \wedge E)} = \frac{P_A(E \wedge H)}{P_A(E )}= P_A^0(H \midd E)
\end{align*}
\end{proof}

\subsection{Prosecutor's conditioning: a justifiable residue of the prosecutor's fallacy}
It is sometimes claimed that reliable reasoning in forensics necessarily requires the 
balancing of at least two scenario's. 
Likelihood ratio transfer represents the communication of an evaluation some form of comparison between two scenario's. 
Transferring a single likelihood 
ratio from MOE to TOF might be considered as a communication concerning merely a single scenario, 
which may be considered problematic for that reason. 
I will show that there is no such problem, at least not in principle. 

The grounds for rejecting one sided reporting of evidence reside in the 
fact that MOE is not supposed to know TOF's prior beliefs. 
For the prosecutor, however, who I will refer to as POC for ``pioneer of claims'', it is acceptable to ask TOF about its prior 
beliefs and
to seek common ground with TOF on that matter in advance of formulating a claim in the form of a strong belief in a 
hypothesis $H$, for instance asserting that a certain course of events took place in a certain manner.
 After having established common ground with TOF concerning shared beliefs joint reasoning may proceed. Here is an example.
 Because there are shared beliefs no subscripts are introduced. If a subscript (holder) to these subjective probabilities must be assigned then 
 that will be TOF.
 
For propositional atoms $D_i$ for $ i \in \underline{n} = \{1,\dots,n\}$ it is assumed that $D_1 \vee \ldots \vee D_n = \top$ and
$D_i \wedge D_j= \bot $ for different $i,j \in \underline{n}$. $D_i$ expresses that focus is on individual $i$. 
Besides the $D_i$'s the proposition space of TOF has generators $E$ and $H$. $E$ satisfies $E = D_1 \vee \ldots \vee D_k$ for some  given $k < n$. 
$H \wedge D_i$ expresses that individual $i$ is considered the unique person (and suspect) who carried out a certain action.
Initially not much is known about $H$, and the overall probability $P(H)$ is rather low, and moreover $H$  is 
unevenly distributed over $k$ individuals, each of  whom attracts some attention as a potential suspect. 
In the example $H$ has a single peak for individual $1$, 
which models an uneven distribution of $H$. 

 In more detail the prior state is that  TOF and POC agree upon the following (shared) 
 beliefs which they both maintain,  including the relative height of $p$ which is motivated by 
 circumstantial evidence indicating individual $1$ as a suspect.
\begin{align*}
P(D_i) & = \frac{1}{n}, &1\leq i \leq n\\
P(H \wedge D_1) &= \frac{p}{n}, & \frac{1}{n}< p \leq 1\\
P(H \wedge D_i) & = \frac{1-p}{n-1}\cdot \frac{1}{n}, &1< i \leq n\\
P(H)&= \frac{1}{n}&\\
P(E)&= \frac{k}{n}\\
P(E \wedge D_i)&= \frac{1}{n}&1\leq i \leq k\\
P(E \wedge D_i)&= 0&k\leq i \leq n\\
P(E \wedge H)&=\sum_{i \in \underline{n}}\frac{P(E \wedge H \wedge D_i)}{P(D_i)} \cdot P(D_i)=\sum_{i \in \underline{n}}P(E \wedge H \wedge D_i)=\\
&=\sum_{i \in \underline{k}}P(E \wedge H \wedge D_i)=P(E \wedge H \wedge D_1) + \sum_{i =2}^{k}P(E \wedge H \wedge D_i)\\
&= \frac{p}{n}+(k-1)\cdot \frac{1-p}{n-1}\cdot \frac{1}{n}=\frac{1}{n}(p+(1-p)\cdot \frac{k-1}{n-1})
\end{align*}

POC calls MOE for advice and is informed by MOE way of a single likelihood transfer that $P^0_{\mathrm{MOE}}(E \midd H) =1$.
In other words, MOE advises POC to restrict the search for the perpetrator to the members of a group characterised by $E$.
POC agrees and requires of TOF  that they also restrict suspicions to individuals that satisfy $E$.

Using prosecutor's fallacy (see Thompson \& Shumann~\cite{ThompsonS1987}) as a reasoning pattern, 
POC and TOF may now infer $\widehat{P}(H \wedge D_1) \approx 1$. For instance with $p = 1/10, k = 100, n = 100.000$ the following
seems to hold:
$$\widehat{P}(H \wedge D_1) = \frac{P(H \wedge D_1)}{ P(H \wedge E)} =    
\frac{p}{n} \cdot \frac{n}{p+(1-p)\cdot \frac{k-1}{n-1}} =\frac{p}{n} 
	\cdot \frac{100.000}{\frac{1}{10}+(1-\frac{1}{10})\cdot \frac{100-1}{100.000-1}} \geq $$
	$$ = \frac{p}{n}  \cdot 500.000$$ thereby deriving a very high posterior probability that the suspect is the perpetrator. 
	
Alternatively (and without any justification either see Paragraph~\ref{toc}) TOF may upon receiving and adopting the information that 
$P^0(E \midd H) =1$ transpose the
conditional, thus obtaining $\widehat{P^0}(H \midd E) =1$ which together with $\widehat{P^0}(E \midd D_1) =1$ leads to 
$\widehat{P^0}(H \midd D_1) =1$ thereby increasing $\widehat{P_0}(H \wedge D_1)$ to $P_0(D_1)$, which in our example amounts to an increase of a factor 10.

Instead of an application of unjustified reasoning  it is possible to use Adams conditioning
in order to capture the belief revision which TOF and POC may justifiably adopt upon learning that $P^0(E \midd H) =1$.

Application of single likelihood Adams conditioning works as follows in this case:
\begin{align*} 
\widehat{P} &=  P(H \wedge E \wedge \bullet) \cdot \frac{l}{P^0(E \midd H)}+
P(H \wedge \neg E  \wedge \bullet ) \cdot \frac{1-l}{P^0(\neg E \midd H)}+P( \neg H \wedge \bullet) \\
&=P(H \wedge E \wedge \bullet) \cdot \frac{1}{P^0(E \midd H)}+P( \neg H \wedge \bullet)\\ 
\widehat{P} (H \wedge D_1) &=(P(H \wedge E \wedge \bullet) \cdot \frac{1}{P^0(E \midd H)}+P( \neg H \wedge \bullet)(H \wedge D_1)\\
&=P(H \wedge E \wedge D_1) \cdot \frac{1}{P^0(E \midd H)}+P( \neg H \wedge H \wedge D_1)\\
&=P(H  \wedge D_1) \cdot \frac{P(H)}{P(E \wedge H)}\\
&= \frac{p}{n} \cdot \frac{1}{n}\cdot \frac{1}{\frac{1}{n}\cdot (p+(1-p)\cdot \frac{k-1}{n-1})}\\
&= \frac{1}{n} \cdot \frac{p}{p+(1-p)\cdot \frac{k-1}{n-1}}\\
&=P(H \wedge D_1) \cdot \frac{1}{p+(1-p)\cdot \frac{k-1}{n-1}}
\end{align*}
The practical value of this conditioning step appears only when looking at an example. 
With $p = 1/10, k = 100, n = 100.000$ one finds:
\begin{align*} 
 \widehat{P} (H \wedge D_1)&=P(H \wedge D_1) \cdot \frac{1}{1/10+9/10\cdot \frac{99}{99.999}} \geq 
 P(H \wedge D_1) \cdot \frac{1}{1/10+9/10\cdot \frac{100}{100.000}}\geq \\
&  P(H \wedge D_1) \cdot \frac{1}{1/10+9/100}=P(H \wedge D_1) \cdot \frac{100}{19}
\end{align*}
Thus on the basis of a single likelihood obtained from MOE both TOF and POC have significantly increased the belief that
the suspect (person 1) has been the perpetrator: $\displaystyle P^0(H \midd D_1) \approx P^0(H \midd D_1) \cdot 100/19$. 
The use of approximation serves an expository purpose only while Adams conditioning provides precise values for all posterior probabilities.
This form of single likelihood Adams conditioning may be referred to as 
prosecutor conditioning (conditioning on a condition proposed by the prosecutor).

\section{Sender-side aspects of LRTMR}
In the previous chapters $A$ represents the role of TOF who is receiving information concerning likelihoods, likelihood ratio's,
and observed evidence from $B$ who represents MOE. In the setting of LRTMR, $B$ operates as a sender of information.
For $B$  a repertoire of messages to $A$ can be distinguished. 
\begin{enumerate}
\item $<L^0_B(E,H) = l >$ (for a closed rational number expression $l$ with $0< l \leq1$) is the message that the likelihood 
of evidence proposition $E$ with respect to hypothesis proposition $H$ is equal to $l$,
\item $<L^0_B(E,\neg H) = l >$ (for a closed rational number expression $l$ with $0< l \leq1$) is the 
message that the likelihood of evidence proposition $E$ with respect to negated hypothesis proposition $H$ is equal to $l$,
\item $<(L^0_B(E,H) = l \,\&\, L^0_B(E,\neg H) = l^\prime )>$ (for closed rational number expressions $l,l^\prime$ with 
$0< l,l^\prime \leq1$) is the message that the likelihood pair
of evidence proposition $E$ with respect to hypothesis proposition $H$ is equal to $(l,l^\prime)$,
\item $<LR^0_B(E,H,\neg H) = r >$ (for a closed rational number expression $r$ with $0< r \leq1$) is 
the message that the likelihood ratio of evidence proposition $E$ with respect to hypothesis proposition $H$ is equal to $r$,

\item $<P_B(E) = 1>$ is the message that $B$ considers evidence proposition $E$ to be true.
\item $<LR^0_B(E,H,\neg H) = r\,  \&\, P_B(E) = 1>$ is the combined (simultaneous) message that 
includes both $<LR^0_B(E,H,\neg H) = r >$ and $<P_B(E) = 1>$.
\end{enumerate}
For a message $<m>$ actions of the form $snd_{B\to A}(<m>)$ may be performed by agent $B$ resulting in the asynchronous\footnote{%
An asynchronous message may arrive later than it was sent. A synchronous message arrives at the same time.
The price paid for synchrony is that sending a message may be delayed until the intended recipient is  able to receive the message. 
Thread algebra of~\cite{BergstraM2007} can be used to specify the deterministic concurrent cooperation of TOF and MOE with either 
synchrounous or asynchronous message passing.} and 
eventually successful transfer of the message $<m>$ to agent $A$. $B$ may also perform internal actions, 
notably finding out likelihoods in advance of transferral.  $\mathit{find}_B (<L^0_B(E,H) = l>)$ represents the action of $B$ 
coming to the belief that $P^0(E \midd H)$ should be given value $l$. Similarly 
$\mathit{find}_B (<L^0_B(E,\neg H) = l>)$ represents the action of $B$ coming to the understanding that $P^0(E \midd \neg H)$
 should be given credence $l$. The action $\mathit{confirm}(<P_B (E)=1>)$ represents $B$'s becoming aware that $E$ is true.

$B$ can carry out its communicative task towards $A$  in many ways. Below only two options for the behaviour of $B$ will be considered, 
thereby limiting attention to the transfer of a likelihood ratio. 

\subsection{Single message reporting: at odds with subjective probability?}
\label{singlemessageR}
In a single message $B$ may reports to $A$.\footnote{%
It is assumed that a message may also contain  explanatory text, but that part of the content is 
ignored at the level of abstraction envisaged in this paper.} 
The report is transferred by a single asynchronous send action:\footnote{%
In Willis et al.\ \cite{Willis2015} (ENFSI guideline for evaluative reporting in forensic science) 
extensive mention is made of the imperative that a likelihood ratio must be included in the report of a forensic expert.  
It is suggested that evidence is not part of FE reporting.  It is not obvious from this guideline if it advises (in the simplest case) 
MOE to make use of to what I am calling single message reporting.}
\[\mathit{snd}_{B \to A}(<LR^0_B(E,H,\neg H) = r \, \& \,P_\MOE(E) = 1>)\]

This behaviour of $B$ is an abstraction of the case that the expert (MOE) reports in a single document, 
while a subsequent interview by $A$  at best produces clarification, 
the outcome of which is not reflected in a probabilistic transformation. 
Under these assumptions the following conclusions can be drawn:
\begin{enumerate}
\item At the time of sending $B$ is at risk not to be reporting its current  beliefs. Indeed if $P_B(E)=1$, a state of affairs 
which $B$ is confirming and transmitting to $A$, then unavoidably also 
$r= \LR_B(E,H) = 1 $ (unless $P_{B}(H)=0$, which is a marginal case). So the act of 
sending a likelihood ratio to $A$ is redundant in this case, unless $B$ is in fact reporting a mix of current and past beliefs. 
Moreover it is assumed 
that $B$ becoming aware of $P_B(E) = 1$ occurred strictly after
$B$'s internal action of  becoming aware of the value of both likelihoods that make up the likelihood ratio reported as $r$.
\item To the extent that subjective belief theory insists that agents communicate the beliefs they are actually holding,
subjective belief theory cannot explain or justify the behaviour of $B$. 
\item Another theory of belief and probability is required for explaining single message reporting by $B$. For instance the use of temporal logic. 
 However, it must be noticed that the combination of timing with
combination with conditionalization (e.g.\ Jeffrey conditionalization) is a strikingly difficult topic.
Weisberg's paradox (Weisberg~\cite{Weisberg2009} and the analysis of it in Huber~\cite{Huber2014} 
provide an indication of the complications involved.
\end{enumerate}

\subsection{Multiple message reporting:  subjective probability compliance?}
Instead of issuing a single message $B$ may transmit two or more consecutive messages. We consider the case that
the first message contains a likelihood pair 
\[\mathit{snd}_{B \to A}(<L^0_B(E,H) = l \,\&\,L^0_B(E,\neg H) = l^\prime >)\]
 and the second message, to be sent after the first message has been received by $A$, 
  consists of an  an assertion of evidence:
  \[\mathit{snd}_{B \to A}(<P_B(E) = 1>)\]
It is understood that when the first message is sent, $B$ has not yet become aware that $P_B(E) = 1$. 
So that in both cases $B$ is reporting 
consistently with their actual beliefs.
  
 At the receiving side $A$ can process the first message from $B$ by way of a double 
 likelihood Adams transformation, 
 and upon receiving the second incoming message  
 from $B$, $A$  can  proceed with Bayesian conditioning without proposition kinetics.
 
 Under the assumption that (as understood from the perspective subjective probability theory) 
 an update of a likelihood pair as received from $B$ is adequately 
 reflected by means of a corresponding double likelihood Adams transformation, 
 it is the case that double message reporting as described above
  is compliant at both sides with the requirements of subjective probability theory.\footnote{%
On the transfer of likelihood pairs: Robertson, Vignaux \& Berger in~\cite{RobertsonVB2011} (p.\ 447) 
indicate that it has become standard (in paternity cases) that a likelihood ratio is conveyed in 
addition to the underlying likelihood pair.  Morisson \& Enzinger~\cite{MorrisonE2016} 
suggest to distinguish between Bayes factor and likelihood ratio and both notions 
may have disparate relations with the respective underlying pairs.}

\section{Concluding remarks: physical probabilities for MOE }
In  forensic logic there seems to an unbridgeable discrepancy between physical probabilities\footnote{%
In Strevens~\cite{Strevens2012} this terminology  is discussed in some detail. 
Frequentistic probabilities are subsumed under the category of physical probabilities in~\cite{Strevens2012}.} 
and subjective probabilities. Nevertheless, mathematically speaking these two mechanisms for measuring uncertainty are quite similar. 

Probability has both subjective and frequentistic (physical) aspects. There is no reason 
to expect or to wish that one of these two views (or any of the views in between that have been developed thus far) would, 
on the basis of philosophical reflection alone, emerge as a dominant perspective on the concept of probability.

In forensic logic, and more generally wherever LRTMR  may apply, 
there appears to be room for both perspectives on probability. This position  intentionally
leaves room for worries such as formulated in Risinger~\cite{Risinger2012} (p.\ 9)  ``that
likelihood ratios would be guessed because of the permissiveness for substituting opinion for fact
which subjectivism would allegedly grant a forensic expert''. Distrust of the notion of subjective probability may 
promote the idea that a frequentist viewpoint provides a self-explanatory 
conceptual framework. That position, however, is an illusion, 
as is witnessed by the circumstance that even gaining an understanding of the  fair binary coin 
leads to significant theoretical complexity (e.g.\ in Belot~\cite{Belot2013}).
According to some authors physical probabilities may be taken for subjective beliefs, see e.g. 
Weatherson~\cite{Weatherson2016} who details ``Lewis' new principle''.
This assertion rephrases the older Lewis principal principle, which asserts that a chance 
(probability in the sense of a frequency) may be taken for a belief of equal degree. 

It may be claimed that whenever MOE is asked to produce a numerical value, e.g. a likelihood ratio,
 an attempt must be made determine the value in an objective manner, leaving little room for the introduction of subjective beliefs. 
 Then notions of precision and accuracy enter the picture, 
as well as the intuition that some ``real'' quantity is being measured. Such an intuition can be accomodated
by taking for, say, a likelihood ratio $r=LR^0_{MOE}(E,H,\neg H)$  a rational number written in 
decimal notation without repetition, and given within an interval 
$[r -\epsilon, r+ \epsilon]$, with $\epsilon$ a positive rational number. 
The number of decimals of $r$ is a measure of accuracy, while $\epsilon$ measures precision. 
A disadvantage of this approach, however, is that it is not obvious how to apply the various probability transformations 
when intervals instead of precise values are to be dealt with.

Although statistical processing at MOE side of a collection of belief functions may result in a valuable
data reduction, the resulting outputs, such as values within an interval are not obviously compatible with the 
probability calculus at hand. Instead of doing statistics before handing over the data to TOF, 
MOE may transfer to TOF a representative sample of data (for instance a sample of values for the same likelihood ratio)  
so that TOF may itself perform statistical processing in a later stage of its activity. 

\subsection{Multiple LR reporting by MOE}
A compromise between expecting MOE to report a single precise likelihood ratio, and to ask for an interval, 
is to propose MOE to report a 
collection of candidate likelihood ratios rather than a single one. 
It will be expected that the collection contains its average value as a member, so that, when asked a unique
likelihood ratio can be chosen as the preferred precise and definite result. 
Another message type is required. For each $E, H$ and $n \in \mathbb{N}^+$ there is a message type:
\[<LR^0_\MOE(E,H,\neg H)\in\{ r_1, \ldots,r_n\} >\]
The meaning of this message is that according to MOE each of $r_1,\ldots,r_n$ is a plausible value for $LR^0_B(E,H,\neg H)$. 
An implicit message is that TOF may look at this collection in statistical terms. 

When processing a multi-LR message TOF 
will transform its belief according to each of the 
values thereby obtaining a collection of probability functions rather than a single one. If TOF is already maintaining a collection of 
belief functions then each of these is transformed in correspondence with each of the LR values in the message, and a possibly larger collection of
probability distributions results. Working with collections of belief functions rather than with a single belief function comes under several names. As an approach it 
introduces ignorance on top of uncertainty.

\subsection{Imprecise beliefs for TOF}
\label{imprecise}
Remarkably the central tenets of 
 belief revision theory as incorporated in AGM style belief revision have 
not found a noticeable audience in forensics. 
Probabilistic AGM theory, which may be of use in forensics
has been developed in Voorbraak~\cite{Voorbraak1996} and in Suzuki~\cite{Suzuki2005}, and is ready for use. 

Imprecise belief modeled by way of non-singleton representors can be traced back to Keynes (see Weatherson~\cite{Weatherson2002}), 
and features explicitly in~Levi~\cite{Levi1974} with subsequent work in e.g.\ Voorbraak~\cite{Voorbraak1996}, 
Weatherson~\cite{Weatherson2007,Weatherson2015} and Rens~\cite{Rens2016}.\footnote{%
Verbal likelihood ratio scales
(see Marquis et al.~\cite{MarquisEA2016}) seem to constitute an approach based on imprecise values, 
but the authors strongly insist that verbal scales must not be understood or used in  that manner.} 
Imprecise belief is supposed to enable the incorporation of ignorance into a framework primarily meant to deal with uncertainty.  
Biedermann~\cite{Biedermann2015} provides a recent exposition and justification of 
the viewpoint that ignorance is merely a variation on the theme of uncertainty, at least in the context of forensic reasoning.

Nevertheless, the restriction to precise belief states is increasingly considered to constitute a source of 
practical problems by authors in forensic science. 
For instance in Morisson \& Enzinger~\cite{MorrisonE2016} and in 
Sjerps et al.~\cite{SjerpsEA2016} the case is made that a likelihood ratio ought to be reported by an MOE
as a value equipped with resolution and precision. 

\subsection{Necessity of background knowledge management for TOF}
Suppose that $0< P_\TOF(H)<1$ and that MOE sends TOF a single message report:
\[\mathit{snd}_{\MOE \to \TOF}(<LR^0_\MOE(E,H,\neg H) = r \, \& \,P_\MOE(E) = 1>)\]
with $0<r$ to which TOF, for whom $E$ is outside its proposition space, reacts by transforming its belief state to $P^\prime_\TOF$ and in 
particular its belief in $H$ according to the familiar equation:
\[ P^\prime_\TOF(H) =\frac{r \cdot P_\TOF(H)}{1 + (r-1) \cdot P_\TOF(H) } \]

Then it follows from the above assumptions that $0 <P^\prime_\TOF(H) < 1$. 
Moreover it is implied that  $r-1 > (r-1) \cdot P_\TOF(H)$, so that $r >1+ (r-1) \cdot P_\TOF(H)$ and 
therefore $P^\prime_\TOF(H) > P_\TOF(H)$. 

Now one may imagine that the same process is repeated, that is TOF receives
the same message from MOE and once more pursues the same transformation, thus obtaining $P^{\prime\prime}_\TOF(H) > P^\prime_\TOF(H)$.
It follows that by repeatedly conditioning on the same evidence, w.r.t. the same likelihood ratio, 
an increasing sequence of probabilities is found for TOF's belief in $H$ with $1$ as its limit.

This phenomenon must be prevented. There is no other option than that TOF maintains a historic record of the probabilistic transformations
which have led to its current belief state, and makes sure that each of these transformations are sufficiently independent. At the background TOF maintains 
a proposition space which incorporates all propositions that have been used as parameters for a probabilistic transformation. 
This is a crucial knowledge management task, the details of which lie outside the realm probability theory, however.

Perhaps more importantly, one notices that the state of knowledge of TOF cannot exclusively be modeled by means of a single belief state and 
that temporal information enters the scene, including historical information about past beliefs.

\subsection{Incomplete belief states for TOF}
The idea that a belief state $(S_A,P_A)$ for $A$ is a partial belief state rests upon the notion that only some of $A$'s beliefs are represented by propositions
in $S_A$. However, a belief state may be incomplete in another manner if its probability distribution is a partial function. We will call such belief states incomplete. 
Thus partiality and incompleteness are logically independent properties of a belief state. Suppose $(S_A,P^\star_A)$ is an incomplete belief state where the superscript 
of $P^\star_A$ indicates that it may be a partial function. Classically an incomplete belief state corresponds to a belief state with a non-singleton representor: 
$$P^{\star\star}_A =\{P_A\colon S_A \to \mathbb{Q}_0 \midd P^\star_A \subseteq P_A\}$$
Allowing TOF to maintain an incomplete belief state is a special case of working with imprecise probabilities.
It may be the case, however,  that $P^{\star\star}_A = \emptyset$ in which the incomplete belief state may be considered inconsistent. Precisely that situation is a point of departure for quantum logic. The Bell inequalities provide criteria under which an incomplete belief state can or cannot be extended to a complete belief state. 
Following the exposition of de Muynck~\cite{Muynck2002} a connection between such criteria and axioms for probability over a meadow is discussed in~\cite{BergstraP2016}.
Whether or not there is a useful role in forensic logic for incomplete belief states $(S_A,P^\star_A)$ with $P^{\star\star}_A = \emptyset$ remains to be seen.
%

\paragraph{Acknowledgement.}Suggestions made by Kees Middelburg concerning the 
precise definition of various fallacies made me consider Adams conditioning. 
Andrea Haker has been very helpful via our longstanding 
discussions regarding the relevance of a focus on forensic reasoning in the context of 
 teaching forensic science.
I also acknowledge Andrea for temporarily handing over to me, during her period of leave to the 
Amsterdam University of Applied Sciences, the role of program director of the MSc 
Forensic Science at the Faculty of Sciences of the University of Amsterdam. 
I acknowledge Yorike Hartman, now coordinating and organizing  the MSc Forensic Science at UvA, for extensive 
discussions on the myriad of (re)design options for the curriculum of the MSc Forensic Science.
Huub Hardy has provided useful and supportive advice including the suggestion not to be worried 
about the vast distance between the cases dealt with by
forensic science practitioners and the fairly theoretical considerations constituting the focus of my research.
Alban Ponse has been very helpful via his contribution to our joint trial and error prone path towards finding a 
usable presentation of probability calculus in the setting of meadows. 
Alban has also made many useful comments on previous drafts of this paper.

\end{document}